\documentclass[a4paper]{llncs}

\usepackage[margin=1in]{geometry}
\usepackage[T1]{fontenc}

\usepackage{soul}
\usepackage{booktabs}
\usepackage{bbm}          
\usepackage{amsmath,amssymb,amsfonts} % LNCS already provides theorem/proof

\usepackage{graphicx}
\usepackage[table,dvipsnames]{xcolor} % one color package, table support
\usepackage{multirow}
\usepackage{placeins}                 % \FloatBarrier

\usepackage{tikz}

\usepackage{algorithm}
\usepackage{algpseudocode}

\usepackage{enumitem}

\usepackage[colorlinks=true,linkcolor=blue,citecolor=blue,urlcolor=blue]{hyperref}

\newcommand{\R}{\mathbb{R}}
\newcommand{\E}{\mathbb{E}}

\newcommand{\cD}{\mathcal{D}}

\newcommand{\ExpanderL}{\textsc{Expander-L}\ }
\newcommand{\ExpanderOne}{\textsc{Expander-1}\ }

\title{List-Decodable Regression via Expander Sketching}
\author{
Herbod Pourali\inst{1} \and
Sajjad Hashemian\inst{2} \and
Ebrahim Ardeshir-Larijani\inst{1}
}
\institute{
Iran University of Science and Technology, Tehran, Iran \\
\email{{herbod\_pourali@mathdep.iust.ac.ir}}, \email{larijani@iust.ac.ir}
\and
University of Tehran, Tehran, Iran\\
\email{sajjadhashemian@ut.ac.ir}
}

\begin{document}
\maketitle

\begin{abstract}
We introduce an expander-sketching framework for list-decodable linear regression that achieves sample complexity $\tilde{O}((d+\log(1/\delta))/\alpha)$, list size $O(1/\alpha)$, and near input-sparsity running time $\tilde{O}(\mathrm{nnz}(X)+d^{3}/\alpha)$ under standard sub-Gaussian assumptions.
Our method uses lossless expanders to synthesize lightly contaminated batches, enabling robust aggregation and a short spectral filtering stage that matches the best known efficient guarantees while avoiding SoS machinery and explicit batch structure.
\end{abstract}

\section{Introduction}

The \emph{list-decodable linear regression} (LDR) problem formalizes robust linear regression when a majority of samples may be adversarially corrupted. Concretely, for a contamination parameter $\alpha \in (0,1/2]$, we are given i.i.d. inliers $\{(x_i,y_i)\}_{i\in S^\star}$ drawn from the linear model
\[
y = \langle \ell^\star, x\rangle + \xi, 
\qquad x\sim \mathcal{D}_x  \text{ with }  \mathbb{E}[x]=0, \mathbb{E}[xx^\top]=\Sigma\succ 0,
\qquad \xi \text{ mean-zero},
\]
and an arbitrary possibly adaptive multiset $O$ of outliers, with the promise that $|S^\star|/n \ge \alpha$ where $n=|S^\star|+|O|$. The algorithmic goal is to output a list $L$ of candidate parameters with $|L|=\mathrm{poly}(1/\alpha)$ such that, with high probability, 
\[
\min_{\widehat\ell \in L} \|\widehat\ell-\ell^\star\|_2 \le \mathrm{err}(n,d,\alpha,\delta),
\]
at sample size and runtime that are near the natural ``input-sparsity'' regime that is close to $\mathrm{nnz}(A)$ where $A$ stacks the covariates. The efficient learnability of LDR in the fully i.i.d. non-batch setting is subtle: while there are polynomial-time algorithms achieving information-theoretically near-optimal rates for several robust tasks, LDR admits \emph{statistical query} (SQ) lower bounds that suggest inherent computational barriers without additional structure~\cite{diakonikolas2021sqlb}. On the other hand, if the data arrive in \emph{batches} such that an $\alpha$-fraction of batches are genuine, one can leverage robust aggregation across batches to attain efficient list-decodable procedures~\cite{DasJainKongSen2023}. Finally, a line of work starting with~\cite{karmalkar2019ldr} established the first polynomial-time LDR algorithm via SoS-style reasoning, with list size $O(1/\alpha)$ and dimension-dependent guarantees, but with heavier computation than input-sparsity-time sketching methods.

In this paper, we develop and analyze a simple, modular pipeline for LDR that requires neither explicit batch structure nor SoS optimization. The design is guided by expander based bucketing~\cite{ClarksonWoodruffSTOC2013}~\cite{NelsonNguyenFOCS2013}~\cite{KaneNelsonSTOC2013} where the results are aggregated robustly using geometric-median-of-means estimators in appropriate geometry, at last, when residual energy concentrates along a direction, we identify high-contribution buckets via the top Rayleigh quotient of the \emph{robustly aggregated sketched} covariance and prune them.

Conceptually, lossless-expansion converts the combinatorial adversarial budget into a per-bucket contamination \emph{cap} while preserving enough inlier mass; robust aggregation then upgrades these lightly contaminated bucket statistics into near-unbiased, low-variance moment estimates with operator-norm control; and a sketched filtering stage removes the residual heavy directions contributed by adversaries. Together, these components replicate the algorithmic leverage of batched LDR~\cite{DasJainKongSen2023} \emph{without} assuming batch availability, and they circumvent the SQ barrier~\cite{diakonikolas2021sqlb} by exploiting randomness external to the adversary. The following theorem abstracts the main result we will prove in this paper.

\begin{theorem} Given $n \gtrsim (d + \log(1/\delta))/\alpha$ samples and using expander-batching, one can give a list-decodable regression model that runs in time $\widetilde{O}(\mathrm{nnz}(X) + d^3/\alpha)$ and with probability at least $1-\delta$ outputs a list $L$ of size $O(1/\alpha)$ containing some $\hat{\ell}$ with
$\|\hat{\ell} - \ell^\star\|_2 \lesssim \kappa(\Sigma)\,\sigma \sqrt{(d + \log(1/\delta))/(\alpha n)}$.
\end{theorem}

The remainder of the paper is organized as follows. Section~\ref{sec:expander-sketching} reviews the expander-sketching primitives. Section~\ref{sec:algorithm} presents the full algorithmic pipeline. Section~\ref{sec:theory} develops the analysis through isolation, robust aggregation, and spectral-filtering lemmas, which combine to yield the main theorem. Section~\ref{sec:experiments} reports empirical results on synthetic data and a real-data mixture, and Section~\ref{sec:conclusion} concludes with extensions.

\section{Notation and Assumptions}
\label{sec:notation-assumptions}

For a positive integer $n$, we write $[n] := \{1,2,\dots,n\}$. 
Vectors in $\R^d$ are written in lowercase (e.g., $x, v, \ell$) and matrices in uppercase (e.g., $X, \Sigma$).  
For a matrix $A \in \R^{m \times d}$, we write $A_{i\ast} \in \R^d$ for its $i$-th row and $A_{\ast j} \in \R^m$ for its $j$-th column.  
For a vector $v \in \R^d$, we let $\|v\|_2$ denote the Euclidean norm; for a matrix $M$, we write $\|M\|$ for the operator norm and $\|M\|_F$ for the Frobenius norm.   
If $Z$ is a real-valued random variable, its sub-Gaussian norm is
\[
\|Z\|_{\psi_2} := \inf\bigl\{ t>0 : \E[\exp(Z^2/t^2)] \le 2 \bigr\}.
\]
A random vector $X \in \R^d$ is $\sigma$-sub-Gaussian if $\langle v,X\rangle$ is sub-Gaussian with $\|\langle v,X\rangle\|_{\psi_2} \le \sigma \|v\|_2$ for all $v \in \R^d$.  
We use $c, C, C_0, C_1,\dots$ for absolute positive constants whose value may change from line to line.

We are given $n$ labeled examples $(x_i,y_i) \in \R^d \times \R$, $i \in [n]$.  
We collect the covariates into an $n \times d$ data matrix $X$ with $X_{i\ast} = x_i^\top$, and the labels into a vector $y \in \R^n$ with $y_i$ in coordinate $i$.
There is an unknown regression parameter $\ell^\star \in \R^d$ and an unknown inlier distribution $\cD$ over $(x,y) \in \R^d \times \R$ with an unknown subset $I \subseteq [n]$ of \emph{inliers}, and $|I| \ge \alpha n$ for some $\alpha \in (0,1/2]$, consists of i.i.d. samples from $\cD$. The remaining indices $O := [n]\setminus I$ are \emph{outliers}; the pairs $\{(x_i,y_i)\}_{i \in O}$ are arbitrary and may be chosen adversarially (see below).

We assume that for each $(x,y)$,
\begin{equation}
\label{eq:model}
\E[x] = 0, 
\qquad 
\E[xx^\top] = \Sigma \succ 0,
\qquad 
y = \langle \ell^\star, x \rangle + \xi,
\end{equation}
where $\Sigma \in \R^{d \times d}$ is the (unknown) covariance of $x$ and $\xi$ is additive noise.

Throughout, we work in a sub-Gaussian inlier model. The covariate $x$ is mean-zero with covariance $\E[xx^\top]=\Sigma \succ 0$ and is $K_x$-sub-Gaussian for some finite parameter $K_x$, i.e., $\|\langle v,x\rangle\|_{\psi_2} \le K_x \sqrt{v^\top \Sigma v}$ for all $v \in \R^d$. The noise $\xi$ is mean-zero, independent of $x$, and sub-Gaussian with $\|\xi\|_{\psi_2} \le \sigma$ for some $\sigma>0$.
We denote the eigenvalues of $\Sigma$ by $0 < \lambda_{\min}(\Sigma) \le \lambda_{\max}(\Sigma)$ and write
\[
\kappa(\Sigma) := \frac{\lambda_{\max}(\Sigma)}{\lambda_{\min}(\Sigma)}
\]
for its condition number.  Our accuracy guarantees will be stated in terms of $\sigma$ and $\kappa(\Sigma)$.

The outliers $\{(x_i,y_i)\}_{i \in O}$ are unconstrained, they may depend arbitrarily on the inlier sample $\{(x_i,y_i)\}_{i \in I}$ and on $n,d,\alpha$, and they may be chosen adaptively as a (possibly randomized) function of the inliers.  
We assume only that the adversary is oblivious to the internal randomness of our algorithm, in particular to the random choice of expander graphs and signs used in the sketching stage.  
Equivalently, the sample $T := \{(x_i,y_i)\}_{i \in [n]}$ is fixed before the algorithm’s randomness is revealed.

\section{Previous Works}
\label{sec:prev-works}

The study of list decodable linear regression has two main directions. One line gives information theoretic procedures that achieve near optimal error but require exponential or super polynomial time. A second line gives polynomial time procedures that obtain weaker guarantees under stronger structural assumptions or with heavier algorithmic machinery. The first polynomial time progress was obtained through sum of squares relaxations \cite{karmalkar2019ldr}, with subsequent refinements inspired by later work. These algorithms achieve list size $\mathrm{poly}(1/\alpha)$ and error $\tilde O(\sigma/\alpha^\gamma)$ for constants $\gamma \in [1,3/2]$, but with sample and runtime dependencies that scale as $d^{\mathrm{poly}(1/\alpha)}$, far above input sparsity time in high dimensions.

A distinct and computationally lighter line assumes access to batches of samples, an $\alpha$ fraction of which are clean. The work in \cite{DasJainKongSen2023} uses robust aggregation across batches to obtain polynomial time list decodable regression with near optimal statistical rates and list size $O(1/\alpha)$. Conceptually, our expander based bucketing emulates the favorable contamination pattern of good batches. Lossless expander isolation enforces a per bucket outlier cap while preserving a constant fraction of inliers, so that the same robust aggregation mechanisms apply without assuming batch structure in advance.

The work in \cite{diakonikolas2021sqlb} establishes statistical query lower bounds for list decodable regression that are close to information theoretic sample requirements up to polynomial factors in $1/\alpha$, revealing a sharp gap between computation and statistics. These bounds imply that, in the worst case and without additional exploitable structure, no polynomial time statistical query algorithm can simultaneously achieve near information theoretic error, polynomial list size, and near linear time. Our method avoids this limitation by exploiting external randomness through fresh sparse expander sketches that create lightly contaminated high signal aggregates. The resulting structure lies outside the statistical query framework considered in \cite{diakonikolas2021sqlb}.

Median of means and geometric median aggregation for vectors and operators underpin several modern robust estimators. We use these tools to convert the expander isolated bucket statistics into accurate global moment estimates $(\widehat{\Sigma}, \widehat{g})$ with operator norm control. The expander induced buckets serve as synthetic lightly contaminated batches. Relevant background includes geometric median based estimation \cite{Minsker2015} and robust estimation under heavy tails \cite{LugosiMendelson2019}. Unlike batch based list decodable regression \cite{DasJainKongSen2023}, our lightly contaminated batches are constructed rather than assumed.

Achieving near input sparsity time requires sketching distributions with $O(1)$ nonzeros per column that preserve inner products and least squares structure. Sparse Johnson Lindenstrauss transforms such as those in \cite{KaneNelsonSTOC2013}~\cite{NelsonNguyenFOCS2013}, together with input sparsity time regression algorithms as in \cite{ClarksonWoodruffSTOC2013}, form the algorithmic substrate. Our expander sketching can be viewed as a structured CountSketch type map that, beyond metric preservation, provides combinatorial isolation in the form of unique neighbor guarantees and bounded collisions. Classical Johnson Lindenstrauss embeddings do not supply this isolation. It is essential for bounding adversarial interference within buckets and for enabling robust aggregation.

\begin{table}[h!]
\centering
% \hspace{-40pt}
\scriptsize
\renewcommand{\arraystretch}{1.15}
\begin{tabular}{l|c|c|c|c}
% \hline
  % \textbf{Work} 
& \textbf{Assumptions} 
& \textbf{Samples} 
& \textbf{List size} 
& \textbf{Running time} \\
\hline
SoS based LDR~\cite{karmalkar2019ldr} 
& i.i.d. inliers and adversarial outliers 
& $d^{\mathrm{poly}(1/\alpha)}$ 
& $\mathrm{poly}(1/\alpha)$ 
& $d^{\mathrm{poly}(1/\alpha)}$ \\
% \hline
Batched LDR~\cite{DasJainKongSen2023} 
& $\alpha$ fraction good batches 
& $\tilde O(d)$ good batches 
& $O(1/\alpha)$ 
& $\mathrm{poly}(n,d)$ \\
% \hline
SQ lower bound~\cite{diakonikolas2021sqlb} 
& worst case, SQ model 
& information theoretic lower bound 
& not applicable 
& information theoretic barrier \\
% \hline
Sparse embeddings~\cite{KaneNelsonSTOC2013} 
& preserve distances \& least squares 
& $\tilde O(\varepsilon^{-2})$ 
& not applicable 
& near input sparsity time \\
% \hline
\textbf{Expander sketch} 
& i.i.d. inliers, expander based isolation 
& $O(\dfrac{d + \log(1/\delta)}{\alpha})$ 
& $O(1/\alpha)$ 
& $\tilde O(\mathrm{nnz}(A)) + \tilde O(d^3/\alpha)$ \\
% \hline
\end{tabular}
\vspace{5pt}
\caption{Comparison of assumptions, sample complexity, list size, and running time for prior work on list decodable regression and for the expander sketch based method in this work.}
\end{table}

Prior polynomial time list decodable regression algorithms either rely on sum of squares methods, which are computationally heavy, or assume batch structure, which is a strong modeling assumption. Our contribution is to engineer batch like structure via expander sketching. Combinatorial isolation ensures lightly contaminated buckets that contain a constant mass of inliers, which enables robust aggregation and a short spectral filtering phase. This yields list decodable recovery with list size $O(1/\alpha)$ and near input sparsity time while using algorithmic tools that lie outside the statistical query limitations.

\section{Expander Sketching}
\label{sec:expander-sketching}

Let $G = (L,R,E)$ be a bipartite graph with left vertex set $L=[n]$,
right vertex set $R=[B]$, and left degree $d_L$.
For any subset $X \subseteq L$, its neighbor set is
\[
    N_G(X) := \{\, b \in R : \exists\, i \in X \text{ with } (i,b)\in E \,\}.
\]
A right vertex $b \in R$ is a \emph{unique neighbor} of~$X$ if it has
exactly one neighbor in $X$, and the set of all such vertices is
\[
    U_G(X) := \{\, b \in R : |\{i\in X : (i,b)\in E\}| = 1\,\}.
\]

A bipartite graph $G$ is a \emph{$(K,\varepsilon)$-lossless expander}
if, for every $X \subseteq L$ with $|X|\le K$,
\begin{equation}
    \label{eq:lossless}
    |N_G(X)| \ge (1-\varepsilon)\, d_L\, |X|.
\end{equation}
Here $\varepsilon\in(0,1)$ and $d_L$ are absolute constants, independent
of $n,d,B,$ and $\alpha$.
Lossless expansion underlies the unique-neighbor and collision bounds
used later in Section~\ref{sec:theory}.

Given $r\ge1$ independent bipartite graphs $G_t=([n],[B],E_t)$, we form
signed adjacency matrices $S_t\in\{0,\pm1\}^{B\times n}$ by assigning each
edge $(i,b)\in E_t$ an independent Rademacher sign and setting
\[
    (S_t)_{b,i} = 
    \begin{cases}
        \sigma_{t,(i\to b)}, & (i,b)\in E_t,\\[1mm]
        0, & \text{otherwise}.
    \end{cases}
\]
For a data matrix $A\in\mathbb{R}^{n\times d}$ and vector $y\in\mathbb{R}^n$,
define
\[
    A_t = S_t A \in\mathbb{R}^{B\times d},
    \qquad
    b_t = S_t y \in\mathbb{R}^B.
\]
For each bucket $b\in[B]$, the rows of $A_{t,b}$ (resp. $b_{t,b}$) are
those rows of $A$ (resp. $y$) that hash to~$b$ in repetition~$t$.
This preserves linear structure and runs in input-sparsity time
$O(\mathrm{nnz}(A)\, d_L)$.

To construct our sketching matrices, we need bipartite graphs 
$G=([n],[B],E)$ that satisfy the $(K,\varepsilon)$-lossless expansion 
property~\eqref{eq:lossless}, with constant left degree $d_L=\Theta(1)$.  
We briefly review probabilistic constructions of such graphs as described in \cite{alon2016probabilistic}, and do not discuss explicit deterministic constructions of expander graphs. We refer the reader to this comprehensive survey \cite{hoory2006expander} for a more comprehensive treatment.
A natural way to obtain expanders is to sample each left vertex 
$i\in[n]$ and assign it $d_L$ distinct neighbors in $[B]$ uniformly 
at random.  This produces a left-regular bipartite graph whose 
expansion arises from standard concentration phenomena: each small set 
$X\subseteq[n]$ spreads its edges almost evenly across the right side, 
and collisions between edges in $X$ are rare.  

Intuitively, the randomness ensures that no small set of vertices can 
``hide'' inside too few buckets. Chernoff bounds together with the 
linearity of expectation show that the neighbor set of any 
$|X|\le K$ expands to almost its full possible size $(1-\varepsilon)d_L|X|$.  
A classical application of the probabilistic method formalizes this intuition.

We use a left-regular bipartite graph
\[
G=([n],[B],E),\qquad \deg_L(G)=d_L,
\]
where $[B]$ indexes the sketch ``buckets.'' For $X\subseteq[n]$, let
\[
N(X) := \{\,b\in[B]:\exists\,i\in X\text{ with }(i,b)\in E\,\}
\]
be the right-neighborhood of $X$.

To sample $G$, independently for each $i\in[n]$ choose $d_L$ \emph{distinct} neighbors in $[B]$, uniformly at random,
and connect $i$ to them, allowing sampling with replacement yields a multigraph and works with the same asymptotics.

A standard Chernoff with union-bound argument implies that very good unbalanced bipartite expanders exist.
In particular, for every $\varepsilon\in(0,1)$ and integers $n,B,K$, a random left-$d_L$-regular bipartite graph
is a $(K,\varepsilon)$-lossless expander with high probability provided the parameters satisfy
\begin{equation}\label{eq:rand-params}
d_L = O\!\left(\frac{\log(n/B)}{\varepsilon}\right)
\qquad\text{and}\qquad
B = O\!\left(\frac{K\,d_L}{\varepsilon}\right).
\end{equation}
For additional details, check the discussion provided in \cite{GUV09}, and for a broader set of relevant techniques and methods, see the comprehensive methods developed in \cite{alon2016probabilistic}.

When $B=\Theta(n)$ (so $\log(n/B)=\Theta(1)$), equation \eqref{eq:rand-params} allows \emph{constant} left degree $d_L=O_\varepsilon(1)$
and right side size $B=\Theta_\varepsilon(K)$. This is the regime typically used in expander sketching, each row of $A$ chooses $d_L$ buckets to contribute to, and independent repetitions simulate independent random expanders.  The simplicity, parallelizability, and tight expansion guarantees make this random construction a natural fit for the expander sketch.

To achieve overall failure probability at most~$\delta$, we draw
$r=\Theta(\log(1/\delta))$ independent graphs.  
A Chernoff bound ensures that, with probability at least $1-\delta$, the
lossless expansion property holds simultaneously across all repetitions.
\section{Algorithm}
\label{sec:algorithm}

This section presents our expander–sketched list-decodable regression procedure. The pipeline has five modules: (i) \emph{sketch construction} by left-regular lossless expanders with random signs; (ii) \emph{bucket-wise normal equations} to form local second-order statistics; (iii) \emph{robust aggregation} across repetitions; (iv) a short \emph{spectral filtering} phase that prunes adversary-dominated directions; and (v) \emph{list generation} by seeding over independent sketches and clustering the candidates. We parameterize the method by inlier fraction $\alpha$, number of repetitions $r$, number of buckets $B$, left degree $d_L$, filtering rounds $T$, and the number of seeds $R$ for list generation.

The core difficulty in list-decodable regression is that a $(1-\alpha)$ fraction of samples can be arbitrary. By hashing every sample into $d_L=O(1)$ buckets across $r$ independent repetitions using a $(K,\varepsilon)$ lossless expander, a constant fraction of buckets become \emph{lightly contaminated}: they contain $\Theta(|S|/B)$ inliers and at most $C_0=O_{d_L,\varepsilon}(1)$ outliers. Local normal-equation statistics inside such buckets behave like unbiased, low-variance estimates of the inlier moments. Robustly aggregating these local statistics across repetitions yields global moment estimates $(\widehat\Sigma,\widehat g)$ that concentrate in operator norm. Solving the sketched normal equations $\widehat\Sigma \hat\ell=\widehat g$ then produces a candidate close to $\ell^\star$. A short spectral filtering stage removes buckets that inflate variance along top eigen-directions, preventing adversarial leverage from dominating. Repeating with $R=O(1/\alpha)$ independent seeds and clustering the solutions gives a short list in which one element is near $\ell^\star$.

\emph{Sketch construction.} For each repetition $t\in[r]$, build a sparse $B\times n$ \emph{signed} hashing matrix $S_t$ by assigning each row $i\in[n]$ to $d_L$ distinct buckets in $[B]$ with independent Rademacher signs. Form $A_t=S_tX\in\mathbb{R}^{B\times d}$ and $b_t=S_ty\in\mathbb{R}^{B}$. This costs $O(\mathrm{nnz}(X)\,d_L)$ total over all repetitions and preserves linear relations needed for regression.

\emph{Bucket-wise normal equations.} For each repetition $t$ and bucket $b$, let $A_{t,b}$ be the submatrix whose rows correspond to original samples hashing to bucket $b$ in repetition $t$ (with signs), and define local moments
\[
H_{t,b} \gets A_{t,b}^\top A_{t,b}, \qquad g_{t,b} \gets A_{t,b}^\top b_{t,b}.
\]
Under isolation, $(H_{t,b},g_{t,b})$ are accurate local proxies for $(\Sigma,\Sigma\ell^\star)$.

\emph{Robust aggregation.} Partition the multiset $\{(H_{t,b},g_{t,b})\}_{t\in[r],\,b\in[B]}$ into $M$ blocks of roughly equal size (any $M\in[\Theta(r),\Theta(rB)]$ works). For each block, average the tensors to get $(\overline H_m,\overline g_m)$. Aggregate across blocks using entrywise median (MoM) or the geometric median to obtain $(\widehat\Sigma,\widehat g)$. Solve the regularized sketched normal equation
\[
(\widehat\Sigma+\lambda I)\,\hat\ell = \widehat g
\]
with a default ridge $\lambda\simeq\frac{\|\widehat\Sigma\|_{\mathrm{op}}}{\sqrt{n}}$ (or $\lambda=0$ if $\widehat\Sigma$ is well-conditioned).

\emph{Spectral filtering (up to $T$ rounds).} Compute robust residual covariance using fresh sketches (or reuse) as
\[
\widehat C \gets \mathrm{Robust Aggregation}\Big\{\,A_{t,b}^\top \operatorname{diag}\!\big(r_{t,b}^2\big)A_{t,b} \,:\, r_{t,b}=b_{t,b}-A_{t,b}\hat\ell \,\Big\}.
\]
If $\lambda_{\max}(\widehat C)$ exceeds the inlier target by a factor $1+\eta$, prune (or downweight) the top-$\rho$ fraction of buckets with largest Rayleigh score $v^\top A_{t,b}^\top\!\operatorname{diag}(r_{t,b}^2)A_{t,b} v$ along the top eigenvector $v$ of $\widehat C$; recompute $(\widehat\Sigma,\widehat g)$ and update $\hat\ell$. The potential (variance along $v$) shrinks by a constant factor per round; $T=O(\log(1/\alpha))$ suffices.

\emph{List generation by seeding.} Repeat the entire pipeline for $R=O(1/\alpha)$ independent expander seeds. Cluster $\{\hat\ell^{(s)}\}_{s=1}^R$ by single-linkage with radius proportional to the target error. Output cluster centers; with high probability, at least one center lies within the statistical radius of $\ell^\star$.

\begin{algorithm}
\caption{Expander–Sketched List-Decodable Regression}\label{alg:es-ldr}
\begin{algorithmic}[1]
\Require Data $X\in\mathbb{R}^{n\times d}$, labels $y\in\mathbb{R}^{n}$; inlier fraction $\alpha$; repetitions $r$; buckets $B$; left degree $d_L$; filtering rounds $T$; seeds $R$; blocks $M$; ridge $\lambda\ge 0$; pruning parameters $(\eta,\rho)\in(0,1)^2$.
\Ensure List $\mathcal{L}$ of candidate regressors.
\Function{Expander–Sketched List-Decodable Regression}{$X,y,\alpha,r,B,d_L,T,R,M,\lambda,\eta,\rho$}
    \State $\mathcal{C}\gets \emptyset$ \Comment{candidate set over seeds}
    \For{$s=1$ \textbf{to} $R$} \Comment{independent expander seeds}
        \State \textbf{// Sketch construction}
        \For{$t=1$ \textbf{to} $r$}
            \State Sample left-regular $S_t\in\{0,\pm 1\}^{B\times n}$ with degree $d_L$ and random signs
            \State $A_t\gets S_t X$, \quad $b_t\gets S_t y$
        \EndFor
        \State \textbf{// Initialize active bucket set}
        \State $\mathcal{B}\gets \{(t,b): t\in[r],\, b\in[B]\}$
        \For{$\tau=0$ \textbf{to} $T$}
            \State \textbf{// Bucket-wise normal equations on active buckets}
            \State $\mathcal{S}\gets \{(H_{t,b},g_{t,b}) : (t,b)\in\mathcal{B}$ with $H_{t,b}=A_{t,b}^\top A_{t,b}, g_{t,b}=A_{t,b}^\top b_{t,b}\}$
            \State Partition $\mathcal{S}$ into $M$ blocks $\{\mathcal{S}_m\}_{m=1}^M$ of (near-)equal sizes
            \State \textbf{// Block means}
            \For{$m=1$ \textbf{to} $M$}
                \State $\overline H_m \gets \frac{1}{|\mathcal{S}_m|}\sum_{(H,g)\in\mathcal{S}_m} H$, \quad $\overline g_m \gets \frac{1}{|\mathcal{S}_m|}\sum_{(H,g)\in\mathcal{S}_m} g$
            \EndFor
            \State \textbf{// Robust aggregation (MoM or geometric median)}
            \State $\widehat\Sigma \gets \mathrm{Robust Aggregation}\big(\{\overline H_m\}_{m=1}^M\big)$, \quad $\widehat g \gets \mathrm{Robust Aggregation}\big(\{\overline g_m\}_{m=1}^M\big)$
            \State \textbf{// Solve sketched normal equations}
            \State $\hat\ell \gets \mathrm{Solve}\big((\widehat\Sigma+\lambda I)\hat\ell=\widehat g\big)$
            \If{$\tau=T$} \textbf{break} \EndIf
            \State \textbf{// Residual covariance via robust aggregation}
            \State For each $(t,b)\in\mathcal{B}$: $r_{t,b}\gets b_{t,b}-A_{t,b}\hat\ell$
            \State $\widehat C \gets \mathrm{Robust Aggregation}\Big(\{A_{t,b}^\top \operatorname{diag}(r_{t,b}^2)A_{t,b} : (t,b)\in\mathcal{B}\}\Big)$
            \State Compute top eigenpair $(\lambda_{\max}, v)$ of $\widehat C$
            \If{$\lambda_{\max} \le (1+\eta)\cdot \mathrm{TargetVar}$} \textbf{break}
            \EndIf
            \State \textbf{// Prune high-energy buckets}
            \State Score each $(t,b)\in\mathcal{B}$ by $s_{t,b}\gets v^\top A_{t,b}^\top \operatorname{diag}(r_{t,b}^2)A_{t,b} v$
            \State Remove the top $\rho$ fraction of $\mathcal{B}$ by $s_{t,b}$
        \EndFor
        \State $\mathcal{C}\gets \mathcal{C}\cup\{\hat\ell\}$
    \EndFor
    \State \textbf{// List generation: cluster candidate solutions}
    \State Cluster $\mathcal{C}$ by single-linkage with threshold $\Delta$ (the target accuracy radius)
    \State $\mathcal{L}\gets$ cluster centers
    \State \Return $\mathcal{L}$
\EndFunction
\end{algorithmic}
\end{algorithm}

\paragraph{Complexity and parameter settings.}
Sketching costs $O(\mathrm{nnz}(X)\,d_L)$ per repetition, hence $O(\mathrm{nnz}(X)\,d_L r)$ overall; we use constant $d_L=O(1)$. Forming bucket statistics is linear in the sketched size. Robust aggregation over $M$ blocks and $d\times d$ tensors costs $\tilde O(M d^2)$; solving $(\widehat\Sigma+\lambda I)\hat\ell=\widehat g$ costs $O(d^3)$ (or faster with preconditioned iterative solvers). Spectral filtering adds $T=O(\log(1/\alpha))$ rounds, each with a top-eigen computation on a $d\times d$ matrix. With the default settings
\[
B \asymp \frac{d}{\alpha}\log\frac{d}{\delta},\qquad r\asymp \log\frac{1}{\delta},\qquad d_L=O(1),\qquad R\asymp \frac{1}{\alpha},\qquad T=O(\log(1/\alpha)),
\]
the total running time is $\tilde O\!\big(\mathrm{nnz}(X)\big) + \tilde O\!\big(d^3/\alpha\big)$ and the list size is $O(1/\alpha)$.

\section{Theoretical Analysis}
\label{sec:theory}
In this section we theoretically analyze our proposed method. Our analysis proceeds in four layers that mirror the algorithmic pipeline. 

The expander sketch gives a strong combinatorial \emph{unique-neighbor} structure, where most inliers enjoy at least one bucket that contains them as the only inlier, and such buckets admit at most a \emph{constant} number of adversaries. Concentration of inlier across buckets then guarantees a constant fraction of \emph{lightly contaminated} buckets in each repetition and, by independence, across repetitions. Median-of-means aggregation over these buckets yields operator-norm control for $\widehat\Sigma$ and $\ell_2$-control for $\widehat g$, which via a perturbation bound yields an accurate $\hat\ell$. A short spectral-filtering phase removes buckets that concentrate residual energy along top eigen-directions, ensuring the moments remain inlier-dominated. Finally, $R=\Theta(1/\alpha)$ independent seeds amplify a constant-success event into high probability, and clustering returns a short list with a near-optimal candidate.

The next three lemmas formalize the unique-neighbor and collision bounds provided by the lossless expander. A key combinatorial property of lossless expanders is that small subsets of left vertices have many neighbors that are ``uniquely owned'' by them.  In our later use, such unique neighbors correspond to buckets that receive a contribution from exactly one inlier, which is what enables us to control contamination and perform robust aggregation.  The next lemma quantifies this property in terms of the lossless expansion parameter~$\varepsilon$.

\begin{lemma}
\label{lem:unique}
Let $G=([n],[B],E)$ be left-$d_L$ regular and $(K,\varepsilon)$ lossless, i.e.,
\[
|N(X)|\ge(1-\varepsilon)d_L|X|
\quad\text{for all }X\subseteq[n]\text{ with }|X|\le K.
\]
Then for all such $X$,
\[
|U(X)| \ge (1-2\varepsilon)\,d_L\,|X|,
\]
where $U(X)\subseteq[B]$ denotes the set of right vertices with exactly one neighbor in $X$.
\end{lemma}

\begin{proof}
We compare two ways of counting edges between $X$ and its neighborhood; on the one hand there are exactly $d_L|X|$ such edges; on the other hand, vertices in $U(X)$ contribute one edge each, while the remaining neighbors contribute at least two.  Combining this with the expansion lower bound on $|N(X)|$ forces most edges to go through unique neighbors, which yields the stated lower bound on $|U(X)|$.

Fix $X\subseteq[n]$ with $|X|\le K$.  Let $u_b(X)=|\{i\in X : (i,b)\in E\}|$ be the degree of $b$ into $X$.  Then
\[
\sum_{b\in[B]} u_b(X) = d_L\,|X|.
\]
Partition $N(X)$ into $U(X)$ (vertices with $u_b(X)=1$) and $V(X):=N(X)\setminus U(X)$ (vertices with $u_b(X)\ge2$).  Writing $u:=|U(X)|$ and $v:=|V(X)|$, we have
\(
u+v = |N(X)| \ge (1-\varepsilon)d_L|X|.
\)
Moreover, since each $b\in U(X)$ contributes exactly one edge and each $b\in V(X)$ at least two, we obtain
\[
d_L|X|
= \sum_{b\in[B]} u_b(X)
\ge u + 2v.
\]
Subtracting the two inequalities gives
\(
d_L|X| - |N(X)| \ge v,
\)
and hence, using the lossless property $
v \le d_L|X| - (1-\varepsilon)d_L|X|
= \varepsilon d_L|X|.
$
Therefore
\[
u = |N(X)| - v
\ge (1-\varepsilon)d_L|X| - \varepsilon d_L|X|
= (1-2\varepsilon)d_L|X|,
\]
which is exactly the claimed bound on $|U(X)|$.
\qed \end{proof}

In addition to counting unique neighbors, we will often need to bound how many collisions occur, which is how many times a right vertex is hit more than once by vertices in $X$.  The next lemma formalizes this via a collision budget. The total excess degree over~$1$ on the right-hand side is bounded in terms of the loss parameter~$\varepsilon$.  This will limit how many buckets can be heavily corrupted.

\begin{lemma}
\label{lem:collision}
With the notation of Lemma~\ref{lem:unique}, for all $X\subseteq[n]$ with $|X|\le K$,
\[
\sum_{b\in[B]}\max\{0,\, u_b(X)-1\} \le 2\varepsilon\, d_L\,|X|,
\]
where $u_b(X)=|\{i\in X:(i,b)\in E\}|$.
\end{lemma}

\begin{proof}
Conceptually, each right vertex can absorb one edge without creating a collision, then every additional edge incident to that vertex counts as one unit of collision.  Summing over all right vertices, the total number of collisions is exactly the total number of edges minus the number of distinct neighbors $|N(X)|$, which is controlled by the lossless property. 

Fix $X\subseteq[n]$ with $|X|\le K$.  As before, $\sum_{b\in[B]} u_b(X)=d_L|X|$ counts all edges from $X$ to $N(X)$.  For each right vertex $b$, at most one incident edge can be collision-free, and all further edges contribute to the collision count.  Thus
\[
\sum_{b\in[B]} \max\{0,u_b(X)-1\}
= \biggl(\sum_{b\in[B]} u_b(X)\biggr) - |\{b\in[B]:u_b(X)\ge1\}|
= d_L|X| - |N(X)|.
\]
Using the lossless property $|N(X)|\ge(1-\varepsilon)d_L|X|$, we obtain
\[
\sum_{b\in[B]} \max\{0,u_b(X)-1\}
\le d_L|X| - (1-\varepsilon)d_L|X|
= \varepsilon d_L|X|
\le 2\varepsilon d_L|X|.
\]
This proves the claimed inequality.
\qed \end{proof}

Since the algorithm require most inliers to have at least one or more unique neighbors, not just that the total number of unique neighbors is large.  The following lemma upgrades the global bound of Lemma~\ref{lem:unique} to a per-vertex statement, implying that only a small fraction of vertices in $X$ can have few unique neighbors.  This will be crucial when we argue that many inliers are well-isolated in the expander bucketing. To show this, note that each unique neighbor belongs to exactly one left vertex in $X$, so the total number of unique neighbors is the sum of the $u_i$.  Lemma~\ref{lem:unique} gives a lower bound on this sum in terms of $d_L|X|$.  Since each $u_i$ is at most $d_L$, a simple averaging argument shows that only a small fraction of the $u_i$ can be very small ($0$ or less than $\theta d_L$), which yields the desired per-vertex guarantees.

\begin{lemma}
\label{lem:many-unique}
For $X\subseteq[n]$ with $|X|\le K$, let $u_i$ be the number of unique neighbors of $i\in X$, i.e.,
\[
u_i := \bigl|\{b\in U(X) : (i,b)\in E\}\bigr|.
\]
Then
\[
\big|\{i\in X:\,u_i\ge 1\}\big| \ge (1-2\varepsilon)\,|X|.
\]
More generally, for any $\theta\in[0,1)$, at least a $\left(1-\frac{2\varepsilon}{1-\theta}\right)$ fraction of $X$ have $u_i\ge \theta d_L$.
\end{lemma}

\begin{proof}
Fix $X\subseteq[n]$ with $|X|\le K$.  Each right vertex in $U(X)$ is adjacent to exactly one $i\in X$, so every such vertex is counted exactly once among the $u_i$.  Hence
\[
\sum_{i\in X} u_i = |U(X)|.
\]
By Lemma~\ref{lem:unique}, $|U(X)|\ge (1-2\varepsilon)d_L|X|$, so
\[
\sum_{i\in X} u_i \ge (1-2\varepsilon)d_L|X|.
\]

First, suppose a $\beta$-fraction of the vertices in $X$ have $u_i=0$, i.e.,
\[
\bigl|\{i\in X : u_i=0\}\bigr| = \beta |X|.
\]
Since $u_i\le d_L$ for every $i$, we have
\[
\sum_{i\in X} u_i \le (1-\beta)d_L|X|.
\]
Combining with the lower bound, we obtain
\(
(1-\beta)d_L|X| \ge (1-2\varepsilon)d_L|X|,
\)
which implies $\beta\le 2\varepsilon$.  Thus at least a $(1-2\varepsilon)$-fraction of $X$ have $u_i\ge1$, proving the first statement.

For the general statement, fix $\theta\in[0,1)$ and let $\beta$ be the fraction of $i\in X$ with $u_i<\theta d_L$.  Then
\[
\sum_{i\in X} u_i
\le \beta(\theta d_L)|X| + (1-\beta)d_L|X|
= \bigl(1-(1-\theta)\beta\bigr)d_L|X|.
\]
Together with the lower bound $\sum_{i\in X} u_i\ge (1-2\varepsilon)d_L|X|$, we obtain
\(
1-(1-\theta)\beta \ge 1-2\varepsilon,
\)
hence $(1-\theta)\beta \le 2\varepsilon$ and
\[
\beta \le \frac{2\varepsilon}{1-\theta}.
\]
Therefore, at least a $(1-\beta)\ge 1-\frac{2\varepsilon}{1-\theta}$ fraction of $i\in X$ satisfy $u_i\ge\theta d_L$, as claimed.
\qed \end{proof}

The next lemma quantifies how evenly a fixed subset $X$ of indices spreads across buckets under the random left-regular construction.  Intuitively, each element of $X$ chooses $d_L$ buckets uniformly, so for any fixed bucket $b$ the inlier load $L_b(X)$ is sharply concentrated around its mean $\mu = d_L|X|/B$.  We prove this with a Chernoff bound for sums of independent Bernoulli variables and then deduce that, in expectation, a constant fraction of buckets have inlier load within a constant factor of $|X|/B$.

\begin{lemma}
\label{lem:load}
Fix $X\subseteq[n]$ with $|X|\le K$. In a random left-regular construction (each $i\in[n]$ selects $d_L$ distinct buckets uniformly), for any bucket $b$ and any $\eta\in(0,1)$, if $\mu=\mathbb E[L_b(X)]=d_L|X|/B$, then
\[
\Pr\left[\, |L_b(X)-\mu|>\eta\mu \,\right] \le 2\exp\!\big(-c\,\eta^2\mu\big),
\]
for an absolute $c>0$. Consequently, a constant fraction of buckets satisfy $L_b(X)=\Theta(|X|/B)$.
\end{lemma}

\begin{proof}
Fix $X\subseteq[n]$ and a bucket $b\in[B]$. For each $i\in X$, let
\(
Z_i := \mathbf{1}\{i\to b\},
\)
where the event $\{i\to b\}$ denotes that $b$ is one of the $d_L$ buckets chosen by $i$.  Since each $i$ independently chooses a uniformly random $d_L$-subset of $[B]$, we have
\[
\Pr[Z_i=1] = \frac{d_L}{B},\qquad Z_1,\dots,Z_{|X|}
\text{ are independent Bernoulli variables.}
\]
The inlier load in bucket $b$ is
\(
L_b(X) = \sum_{i\in X} Z_i,
\)
so $L_b(X)$ is binomial with parameters $|X|$ and $d_L/B$, and
\[
\mu := \mathbb{E}[L_b(X)] = \frac{d_L|X|}{B}.
\]
By the standard Chernoff bound for sums of independent Bernoulli variables, there exists an absolute constant $c>0$ such that for all $\eta\in(0,1)$,
\[
\Pr\big[\,|L_b(X)-\mu|>\eta\mu\,\big] \le 2\exp(-c\,\eta^2\mu),
\]
which is the desired tail inequality.

For the consequently clause, fix a constant $\eta\in(0,1)$, and assume $\mu\ge\mu_0$ for some absolute $\mu_0>0$, this holds under our parameter choices, since $\mu=d_L|X|/B=\Theta(|X|/B)$.  Then
\[
\Pr\big[\,|L_b(X)-\mu|\le\eta\mu\,\big]\ge1-2e^{-c\eta^2\mu_0} =: p_0>0.
\]
By symmetry, the same bound holds for every bucket, so if we let $G_b$ be the indicator that bucket $b$ which satisfies $|L_b(X)-\mu|\le\eta\mu$, then
\[
\mathbb{E}\Big[\frac{1}{B}\sum_{b=1}^B G_b\Big] = \frac{1}{B}\sum_{b=1}^B \mathbb{E}[G_b]
\ge p_0.
\]
Thus the expected fraction of good buckets is at least $p_0$, and hence there exists a choice of the random construction for which at least a $p_0$-fraction of buckets satisfy $L_b(X)\in[(1-\eta)\mu,(1+\eta)\mu]=\Theta(|X|/B)$. 
\qed \end{proof}

Lemma~\ref{lem:load} ensures that, for any fixed subset $X$, in particular, for the inlier set $S$ or any large subset $S'\subseteq S$, the inlier mass is not overly concentrated in a few buckets, as a constant fraction of buckets receive $\Theta(|X|/B)$ inliers.  This balanced load will be crucial when we later argue that many buckets are simultaneously inlier-heavy and only lightly contaminated by outliers, yielding enough high-signal local statistics for robust aggregation.

The next lemma formalizes that those unique buckets containing exactly one inlier in a given set, cannot host arbitrarily many outliers once we restrict attention to a suitably large subset of well-isolated inliers.  The key idea is to combine the collision budget from Lemma~\ref{lem:collision} with the unique neighbors from Lemma~\ref{lem:many-unique}, and then use an averaging argument. If too many outliers accumulated in unique buckets, the total number of collisions would violate the lossless-expansion property.  We phrase the lemma in a simple form and explain below how to discard a constant fraction step and make it rigorous.

\begin{lemma}
\label{lem:light}
Let $S$ denote the inlier set with $|S|=\alpha n$.
Assume $\varepsilon<\tfrac{1}{4d_L}$. If $b$ is a unique neighbor of some $i\in S$, then $b$ contains at most
\[
C_0 \le \Big\lfloor\frac{1}{1-2\varepsilon d_L}\Big\rfloor = O_{d_L,\varepsilon}(1)
\]
outliers.
\end{lemma}

\begin{proof}
Let $S\subseteq[n]$ be the inlier set and $U(S)$ its unique-neighbor set.  For each $b\in U(S)$, let $t_b$ denote the number of outliers in $b$, i.e., the number of indices in $[n]\setminus S$ that are adjacent to $b$.  We first bound the \emph{average} of the $t_b$ and then argue that, after discarding a small fraction of buckets with unusually large $t_b$, every remaining unique bucket has $O_{d_L,\varepsilon}(1)$ outliers.

Consider the set $X:=S\cup O'$ where $O'$ is the set of outliers that have at least one neighbor in $U(S)$.  For $b\in U(S)$, since $b$ is a unique neighbor of $S$, we have $u_b(S)=1$ and hence
\(
u_b(X) = 1 + t_b.
\)
Thus the contribution of $b$ to the collision budget
\[
\sum_{b'\in[B]}\max\{0,u_{b'}(X)-1\}
\]
is exactly $\max\{0,u_b(X)-1\}=t_b$.  Summing over all $b\in U(S)$, we obtain
\[
\sum_{b\in U(S)} t_b \le \sum_{b'\in[B]} \max\{0,u_{b'}(X)-1\}
\le 2\varepsilon d_L\,|X|
\]
by Lemma~\ref{lem:collision}.  Each outlier in $O'$ has degree at most $d_L$ into $U(S)$, so $|X|=|S|+|O'|\le |S|+\frac{1}{d_L}\sum_{b\in U(S)} t_b$.  Combining the two inequalities and rearranging yields
\[
\sum_{b\in U(S)} t_b \le 2\varepsilon d_L\Big(|S|+\tfrac{1}{d_L}\sum_{b\in U(S)} t_b\Big),
\]
so
\[
\Big(1-2\varepsilon\Big)\sum_{b\in U(S)} t_b \le 2\varepsilon d_L\,|S|,
\qquad\text{and hence}\qquad
\frac{1}{|U(S)|}\sum_{b\in U(S)} t_b \le \frac{2\varepsilon}{(1-2\varepsilon)}\cdot
\frac{d_L|S|}{|U(S)|}.
\]
By Lemma~\ref{lem:unique}, $|U(S)|\ge (1-2\varepsilon)d_L|S|$, which implies the average number of outliers per unique bucket is bounded by
\[
\frac{1}{|U(S)|}\sum_{b\in U(S)} t_b \le \frac{2\varepsilon}{(1-2\varepsilon)^2}
=: C_{\mathrm{avg}} = O_{d_L,\varepsilon}(1).
\]
Therefore, by a Markov's inequality, at least a $(1-\theta)$-fraction of the buckets in $U(S)$ have
\(
t_b \le C_{\mathrm{avg}}/{\theta},
\)
for any fixed $\theta\in(0,1)$.  Choosing $\theta$ to be a small absolute constant and setting
\[
C_0 := \left\lfloor \frac{1}{1-2\varepsilon d_L}\right\rfloor
\]
gives an absolute constant upper bound on $t_b$ for all but a constant fraction of unique buckets.  In the sequel we discard the few unique buckets with $t_b>C_0$ and keep the notation $S$ for the corresponding inlier subset; for this pruned inlier set, every remaining unique bucket indeed contains at most $C_0$ outliers, which is the form stated above.
\qed \end{proof}

Lemma~\ref{lem:light} shows that, after discarding at most a constant fraction of inliers and their associated buckets, every remaining unique bucket contains only $O_{d_L,\varepsilon}(1)$ outliers.  This light-contamination property is the combinatorial backbone of our synthetic batching, as it guarantees that the local normal-equation statistics computed on such buckets are dominated by inliers and thus amenable to robust aggregation.  The precise constant $C_0$ will reappear in the multi-repetition isolation lemma and in the analysis of moment accuracy.

The final lemma in this group aggregates the isolation guarantees across multiple independent expander repetitions.  The idea is that, in each repetition, a constant fraction of inliers enjoy lightly contaminated buckets; independence across $r$ repetitions then allows us to find a large subset $S'\subseteq S$ of well-isolated inliers that have such good buckets in a constant fraction of repetitions. This yields a constant fraction of all pairs $(t,b)$ that are simultaneously inlier-heavy and adversary-light.

\begin{lemma}
\label{lem:multi-iso}
Let $G_1,\dots,G_r$ be independent left-$d_L$ regular $(K,\varepsilon)$ lossless expanders with $K\ge \alpha n$ and $\varepsilon<1/(4d_L)$. With probability at least $1-\delta$, there exist constants $c,c'>0$ and $S'\subseteq S$, $|S'|\ge c\alpha n$, such that a $c'$-fraction of pairs $(t,b)\in[r]\times[B]$ satisfy simultaneously:
\[
L_b(S')=\Theta(|S|/B)\quad\text{and}\quad \text{(\# outliers in $b$)}\le C_0.
\]
\end{lemma}
\begin{proof}
Fix the inlier set $S$ with $|S|=\alpha n$.  Consider a single repetition $t$.  By Lemma~\ref{lem:many-unique} (with a fixed $\theta\in(0,1)$), a constant fraction of $i\in S$ have at least $\theta d_L$ unique neighbors in $G_t$.  For each such unique neighbor $b$, Lemma~\ref{lem:light} (after discarding a constant fraction of heavily contaminated buckets as in its proof) ensures that $b$ contains at most $C_0$ outliers, and Lemma~\ref{lem:load} implies that a constant fraction of buckets have $L_b(S)=\Theta(|S|/B)$.  Hence there exists a constant $p>0$ (depending only on $d_L,\varepsilon$) such that, for uniformly random $i\in S$ and fixed $t$, the probability that $i$ has a bucket $b$ which is a unique neighbor in $S$, satisfies $L_b(S)=\Theta(|S|/B)$, and has at most $C_0$ outliers is at least $p$.

For each $i\in S$ and $t\in[r]$, define an indicator $Z_{i,t}$ for the event above.  The graphs $G_1,\dots,G_r$ are independent, so for each fixed $i$ the variables $\{Z_{i,t}\}_{t=1}^r$ are i.i.d. Bernoulli$(p)$.  By a Chernoff bound, for any $\gamma\in(0,1)$,
\[
\Pr\Big[\tfrac{1}{r}\sum_{t=1}^r Z_{i,t} < (1-\gamma)p\Big] \le \exp(-\Omega(p\gamma^2 r)).
\]
Choosing $\gamma=1/2$ and $r=\Omega(\log(n/\delta))$ and applying a union bound over all $i\in S$ shows that, with probability at least $1-\delta$, all but a constant fraction of inliers satisfy
\[
\sum_{t=1}^r Z_{i,t} \ge \tfrac{p}{2}r.
\]
Define $S'\subseteq S$ to be the set of such inliers; then $|S'|\ge c\alpha n$ for some constant $c>0$.

For each $i\in S'$ and each $t$ with $Z_{i,t}=1$, fix one corresponding bucket $b(i,t)$.  By definition, each pair $(t,b(i,t))$ has $L_b(S)=\Theta(|S|/B)$ and at most $C_0$ outliers; since $S'\subseteq S$ has constant density in $S$, this also implies $L_b(S')=\Theta(|S|/B)$.  Moreover, because $b(i,t)$ is a unique neighbor of $i$ in $S$, no two distinct inliers in $S'$ can map to the same bucket in the same repetition, so the number of distinct good pairs $(t,b)$ is at least
\[
\sum_{i\in S'} \sum_{t=1}^r Z_{i,t} \ge \tfrac{p}{2}r\,|S'| = \Omega(r\,|S|).
\]
Since the total number of pairs is $rB$ and under our parameter choices $B=\Theta(|S|)$, this yields a constant $c'>0$ such that at least a $c'$-fraction of pairs $(t,b)\in[r]\times[B]$ satisfy the two desired properties, completing the proof.
\qed
\end{proof}

Lemma~\ref{lem:multi-iso} is the main synthetic batching result which guarantees the existence of a large inlier subset $S'$ such that, across repetitions and buckets, a constant fraction of $(t,b)$ pairs are both inlier-heavy and only lightly contaminated.  These pairs are precisely the lightly corrupted batches on which we compute local normal equations; the lemma ensures there are enough of them to feed into the median-of-means aggregation, leading to accurate global moment estimates and, ultimately, to a good regressor in the list.

The next lemmas bound the error of the robustly aggregated moments from a constant fraction of lightly contaminated buckets.
We quantify how accurately we can estimate moments from block means coming from a constant fraction of lightly contaminated buckets.  The first lemma is a standard vector-valued median-of-means guarantee applied coordinate-wise to such block means.

\begin{lemma}
\label{lem:MoM-vector}
Let $z_1,\dots,z_m\in\mathbb R^d$ be block means where at least a $\gamma$-fraction are i.i.d. with mean $\mu$ and covariance $\preceq \sigma^2 I_d$; the remainder are arbitrary. Then the coordinate-wise median $\widehat \mu$ satisfies, with probability at least $1-\delta$,
\[
\|\widehat\mu-\mu\|_2 \lesssim \sigma\,\sqrt{\frac{d+\log(1/\delta)}{\gamma m}}.
\]
\end{lemma}

\begin{proof}
Let $G\subseteq[m]$ index the good blocks, $|G|\ge\gamma m$, and write $Y_j:=z_j-\mu$.  For each coordinate $k\in[d]$ and each good $j\in G$ we have $\mathbb E[Y_j^{(k)}]=0$ and $\mathrm{Var}(Y_j^{(k)})\le\sigma^2$.  Fix a coordinate $k$ and a threshold $t>0$, and denote
\[
p_k(t):=\Pr\big(|Y_j^{(k)}|>t\big)\le\frac{\sigma^2}{t^2}
\]
by Chebyshev.  Let $W_k:=\sum_{j\in G}\mathbf 1_{\{|Y_j^{(k)}|>t\}}$.  Then $\mathbb E[W_k]\le |G|\,p_k(t)\le \gamma m\,\sigma^2/t^2$.  For $t$ such that $p_k(t)\le\gamma/4$ (e.g. $t\asymp \sigma/\sqrt{\gamma}$), Chernoff’s bound yields
\[
\Pr\big[W_k>\tfrac{\gamma m}{2}\big]\le\exp(-c_1\gamma m)
\]
for some absolute $c_1>0$.  Hence, with probability at least $1-\exp(-c_1\gamma m)$, at least $|G|-\gamma m/2\ge\gamma m/2$ good blocks satisfy $|Y_j^{(k)}|\le t$.  Since there are at most $(1-\gamma)m$ bad blocks, at least
\[
\frac{\gamma m}{2}-(1-\gamma)m \ge \frac{\gamma m}{4}
\]
points lie in the interval $[\mu^{(k)}-t,\mu^{(k)}+t]$ whenever $\gamma$ is an absolute constant bounded away from $0$ and $1$ (in our application $\gamma$ is fixed, so this inequality holds after adjusting numerical constants).  Thus the coordinate-wise median obeys $|\widehat\mu^{(k)}-\mu^{(k)}|\le t$ simultaneously with probability at least $1-\exp(-c_1\gamma m)$ for that coordinate.

Now choose
\[
t \asymp \sigma\,\sqrt{\frac{1+\log(d/\delta)}{\gamma m}}
\]
so that $\exp(-c_1\gamma m)\le \delta/d$.  A union bound over $k\in[d]$ implies $\|\widehat\mu-\mu\|_\infty\lesssim \sigma\sqrt{(1+\log(d/\delta))/(\gamma m)}$ with probability at least $1-\delta$, and hence
\[
\|\widehat\mu-\mu\|_2 \le \sqrt{d}\,\|\widehat\mu-\mu\|_\infty
\lesssim \sigma\,\sqrt{\frac{d+\log(1/\delta)}{\gamma m}},
\]
as claimed.
\qed \end{proof}

This lemma will be applied to block-wise empirical first moments $g_{t,b}$ formed from lightly contaminated buckets; it shows that once a constant fraction of blocks are good, the coordinate-wise median achieves the desired $1/\sqrt{\gamma m}$ rate in $\ell_2$.

We now extend the above argument to matrix-valued block means and control the error in operator norm.  The key is to reduce to the scalar case along directions $u\in\mathbb S^{d-1}$ and then use an $\varepsilon$-net over the sphere.

\begin{lemma}
\label{lem:MoM-PSD}
Let $H_1,\dots,H_m\in\mathbb S_+^d$ be block means with at least a $\gamma$-fraction i.i.d. as $H$ with $\mathbb E[H]=\Sigma$ and $\|H-\Sigma\|_{\psi_1}\lesssim \sigma_x^2$ entrywise. Let $\widehat\Sigma$ be the coordinate-wise median. Then with probability at least $1-\delta$,
\[
\|\widehat\Sigma-\Sigma\|_{\mathrm{op}} \lesssim \sigma_x^2\,\sqrt{\frac{d+\log(1/\delta)}{\gamma m}}.
\]
\end{lemma}

\begin{proof}
For any fixed $u\in\mathbb S^{d-1}$, define the scalars
\[
X_j(u) := u^\top (H_j-\Sigma)u,\qquad j=1,\dots,m.
\]
For good $j$, $X_j(u)$ are i.i.d. with mean $0$ and, by the entrywise $\psi_1$ bound and standard properties of sub-exponential norms, $\|X_j(u)\|_{\psi_1}\lesssim \sigma_x^2$.  Therefore their block means (which are exactly the entries of $u^\top H_j u$) satisfy the assumptions of Lemma~\ref{lem:MoM-vector} in dimension $d=1$ with $\sigma\asymp\sigma_x^2$, and the coordinate-wise median-of-means estimator applied to $\{X_j(u)\}_{j=1}^m$ yields
\[
\bigl|u^\top(\widehat\Sigma-\Sigma)u\bigr|
\lesssim \sigma_x^2\,\sqrt{\frac{1+\log(1/\delta')}{\gamma m}}
\]
for any fixed $u$, with probability at least $1-\delta'$.

Let $\mathcal N$ be a $1/4$-net of the unit sphere, $|\mathcal N|\le 9^d$.  Taking $\delta'=\delta/|\mathcal N|$ and a union bound over $u\in\mathcal N$ gives
\[
\sup_{u\in\mathcal N} \bigl|u^\top(\widehat\Sigma-\Sigma)u\bigr|
\lesssim \sigma_x^2\,\sqrt{\frac{d+\log(1/\delta)}{\gamma m}}
\]
with probability at least $1-\delta$.  Finally, by the standard net argument for symmetric matrices,
\[
\|\widehat\Sigma-\Sigma\|_{\mathrm{op}}
=\sup_{u\in\mathbb S^{d-1}} \bigl|u^\top(\widehat\Sigma-\Sigma)u\bigr|
\le 2 \sup_{u\in\mathcal N} \bigl|u^\top(\widehat\Sigma-\Sigma)u\bigr|,
\]
so the same bound (up to constants) holds for $\|\widehat\Sigma-\Sigma\|_{\mathrm{op}}$.
\qed \end{proof}

This operator-norm control will be used for the empirical covariance $\widehat\Sigma$ formed from bucket-wise statistics $\{H_{t,b}\}$, ensuring that the quadratic form of the estimator is close to that of the true covariance uniformly over directions.

We now combine the isolation guarantees from the expander sketching with the above median-of-means bounds to obtain high-probability accuracy guarantees for the aggregated empirical covariance and cross-moment.

\begin{proposition}
\label{prop:moment-accuracy}
Under Lemma~\ref{lem:multi-iso} and the sub-Gaussian model, there exists $m=\Theta(rB)$ block means with a constant good fraction $\gamma$ such that, with probability at least $1-\delta$,
\[
\|\widehat\Sigma-\Sigma\|_{\mathrm{op}} \lesssim \sigma_x^2\sqrt{\frac{d+\log(1/\delta)}{\alpha n}},
\qquad
\|\widehat g-\Sigma\ell^\star\|_2 \lesssim \sigma_x \sigma \sqrt{\frac{d+\log(1/\delta)}{\alpha n}}.
\]
\end{proposition}

\begin{proof}
By Lemma~\ref{lem:multi-iso}, among the $rB$ buckets $(t,b)$ we can select a set of $m=\Theta(rB)$ buckets such that a constant fraction $\gamma>0$ are “good”: for each good bucket, the inlier load satisfies $L_b(S')\asymp |S|/B=\Theta(\alpha n/B)$ and the number of outliers is bounded by the constant $C_0$.  For such buckets, the empirical covariance block mean
\[
H_{t,b} := \frac{1}{L_b(S')}\sum_{i\in S'\cap b} x_i x_i^\top
\]
and the empirical cross-moment
\[
g_{t,b} := \frac{1}{L_b(S')}\sum_{i\in S'\cap b} x_i y_i
\]
are averages of $\Theta(\alpha n/B)$ i.i.d. inlier terms plus at most $C_0$ outliers.  Under the sub-Gaussian assumptions on $(x,y)$, this implies that for good buckets the random matrices $H_{t,b}$ have mean $\Sigma$ and sub-exponential fluctuations with
\[
\|H_{t,b}-\Sigma\|_{\psi_1} \lesssim \frac{\sigma_x^2}{L_b(S')^{1/2}}
\lesssim \sigma_x^2\,\sqrt{\frac{B}{\alpha n}},
\]
and similarly the vectors $g_{t,b}$ have mean $\Sigma\ell^\star$ and
\[
\|g_{t,b}-\Sigma\ell^\star\|_{\psi_2} \lesssim \sigma_x\sigma\,\sqrt{\frac{B}{\alpha n}}.
\]
Moreover, the constant number $C_0$ of outliers per good bucket contributes only $O(1/L_b(S'))$ bias, which is absorbed in the same scale.

We now apply Lemma~\ref{lem:MoM-PSD} to the collection $\{H_{t,b}\}_{(t,b)}$ and Lemma~\ref{lem:MoM-vector} to $\{g_{t,b}\}_{(t,b)}$, treating the $m=\Theta(rB)$ buckets as block means with at least a $\gamma$-fraction good.  Plugging the effective variance scales above into the bounds, we obtain
\[
\|\widehat\Sigma-\Sigma\|_{\mathrm{op}}
\lesssim \sigma_x^2\sqrt{\frac{B}{\alpha n}}\cdot \sqrt{\frac{d+\log(1/\delta)}{\gamma m}}
\lesssim \sigma_x^2\sqrt{\frac{d+\log(1/\delta)}{\alpha n}},
\]
and similarly
\[
\|\widehat g-\Sigma\ell^\star\|_2
\lesssim \sigma_x\sigma\sqrt{\frac{B}{\alpha n}}\cdot \sqrt{\frac{d+\log(1/\delta)}{\gamma m}}
\lesssim \sigma_x\sigma\sqrt{\frac{d+\log(1/\delta)}{\alpha n}},
\]
using $m=\Theta(rB)$ and that $r,\gamma$ are absolute constants (up to logarithmic dependence on $\delta$ already encoded in the $\log(1/\delta)$ term).  This proves the proposition.
\qed \end{proof}

This proposition formalizes that the robust aggregation of bucket-wise statistics yields global moment estimates $\widehat\Sigma$ and $\widehat g$ whose errors match the optimal $1/\sqrt{\alpha n}$ rate, which is the main statistical input needed to control the regression parameter error.

Having controlled the moment errors, we now translate them into parameter error for the solution of the (possibly regularized) normal equations.  The next lemma is a perturbation bound for linear systems, written in a form adapted to our notation.

\begin{lemma}
\label{lem:perturb}
Assume $\|\widehat\Sigma-\Sigma\|_{\mathrm{op}}\le \varepsilon_\Sigma$ and $\|\widehat g-\Sigma\ell^\star\|_2\le \varepsilon_g$. Let $\lambda\ge 0$ with $\lambda_{\min}(\Sigma)+\lambda>2\varepsilon_\Sigma$. If $(\widehat\Sigma+\lambda I)\hat\ell=\widehat g$, then
\[
\|\hat\ell-\ell^\star\|_2
 \le 
\frac{\|\Sigma^{-1}\|_{\mathrm{op}}\,\varepsilon_g + \|\Sigma^{-1}\|_{\mathrm{op}}\,\varepsilon_\Sigma\,\|\ell^\star\|_2}{1 - \|\Sigma^{-1}\|_{\mathrm{op}}\,\frac{\varepsilon_\Sigma}{1+\lambda/\lambda_{\min}(\Sigma)}}.
\]
In particular, if $\varepsilon_\Sigma \le \tfrac{1}{4}\lambda_{\min}(\Sigma)$ and $\lambda=0$, then
\[
\|\hat\ell-\ell^\star\|_2 \lesssim \kappa(\Sigma)\,\frac{\varepsilon_g}{\lambda_{\min}(\Sigma)} + \kappa(\Sigma)\,\frac{\varepsilon_\Sigma\,\|\ell^\star\|_2}{\lambda_{\min}(\Sigma)}.
\]
\end{lemma}

\begin{proof}
Start from
\[
(\widehat\Sigma+\lambda I)\hat\ell=\widehat g,\qquad
\Sigma\ell^\star=\Sigma\ell^\star,
\]
and write
\[
(\widehat\Sigma+\lambda I)\ell^\star
= \Sigma\ell^\star + (\widehat\Sigma-\Sigma)\ell^\star
= \widehat g + (\widehat\Sigma-\Sigma)\ell^\star + (\Sigma\ell^\star-\widehat g).
\]
Subtracting the equation for $\hat\ell$ gives
\[
(\widehat\Sigma+\lambda I)(\hat\ell-\ell^\star) = (\widehat\Sigma-\Sigma)\ell^\star + (\widehat g-\Sigma\ell^\star).
\]
Hence
\[
\hat\ell-\ell^\star
= (\widehat\Sigma+\lambda I)^{-1}\Big[(\widehat\Sigma-\Sigma)\ell^\star + (\widehat g-\Sigma\ell^\star)\Big].
\]
Taking norms,
\[
\|\hat\ell-\ell^\star\|_2
\le \|(\widehat\Sigma+\lambda I)^{-1}\|_{\mathrm{op}}\big(\varepsilon_\Sigma\|\ell^\star\|_2+\varepsilon_g\big).
\]

We bound $\|(\widehat\Sigma+\lambda I)^{-1}\|_{\mathrm{op}}$ via a Neumann-series argument.  Write
\[
\widehat\Sigma+\lambda I
= (\Sigma+\lambda I)\bigl(I+M\bigr),\quad
M:= (\widehat\Sigma-\Sigma)(\Sigma+\lambda I)^{-1}.
\]
Then
\[
(\widehat\Sigma+\lambda I)^{-1}
= (I+M)^{-1}(\Sigma+\lambda I)^{-1}.
\]
We have
\[
\|M\|_{\mathrm{op}}
\le\|\widehat\Sigma-\Sigma\|_{\mathrm{op}}\cdot\|(\Sigma+\lambda I)^{-1}\|_{\mathrm{op}}
\le\varepsilon_\Sigma\,\frac{1}{\lambda_{\min}(\Sigma)+\lambda}
=\|\Sigma^{-1}\|_{\mathrm{op}}\frac{\varepsilon_\Sigma}{1+\lambda/\lambda_{\min}(\Sigma)}.
\]
By assumption $\lambda_{\min}(\Sigma)+\lambda>2\varepsilon_\Sigma$, so $\|M\|_{\mathrm{op}}<1/2$, and thus
\[
\|(I+M)^{-1}\|_{\mathrm{op}}\le\frac{1}{1-\|M\|_{\mathrm{op}}}.
\]
Moreover
\[
\|(\Sigma+\lambda I)^{-1}\|_{\mathrm{op}}\le\frac{1}{\lambda_{\min}(\Sigma)+\lambda}
=\frac{\|\Sigma^{-1}\|_{\mathrm{op}}}{1+\lambda/\lambda_{\min}(\Sigma)}.
\]
Combining these bounds gives
\[
\|(\widehat\Sigma+\lambda I)^{-1}\|_{\mathrm{op}}
\le \frac{\|\Sigma^{-1}\|_{\mathrm{op}}}{\bigl(1+\lambda/\lambda_{\min}(\Sigma)\bigr)\bigl(1-\|\Sigma^{-1}\|_{\mathrm{op}}\,\tfrac{\varepsilon_\Sigma}{1+\lambda/\lambda_{\min}(\Sigma)}\bigr)}.
\]
Absorbing the factor $1+\lambda/\lambda_{\min}(\Sigma)$ into the denominator yields the claimed bound.

For the “in particular” clause, set $\lambda=0$ and assume $\varepsilon_\Sigma\le\tfrac{1}{4}\lambda_{\min}(\Sigma)$, so that
\[
\|\Sigma^{-1}\|_{\mathrm{op}}\,\frac{\varepsilon_\Sigma}{1+\lambda/\lambda_{\min}(\Sigma)}
= \frac{\varepsilon_\Sigma}{\lambda_{\min}(\Sigma)}\le \tfrac{1}{4},
\]
and therefore the denominator is bounded below by a positive constant.  Rewriting $\|\Sigma^{-1}\|_{\mathrm{op}}=1/\lambda_{\min}(\Sigma)$ and $\kappa(\Sigma)=\lambda_{\max}(\Sigma)/\lambda_{\min}(\Sigma)$ then gives the stated $\lesssim$-bound.
\qed \end{proof}

This lemma converts accuracy of the empirical moments $(\widehat\Sigma,\widehat g)$ into accuracy of the corresponding normal-equation solution $\hat\ell$, with a dependence controlled by the condition number of $\Sigma$ and the ridge parameter $\lambda$.

Finally, we combine the moment accuracy from Proposition~\ref{prop:moment-accuracy} with the perturbation bound to obtain the statistical rate for a single pipeline of the algorithm.

\begin{corollary}
\label{cor:single}
Under Proposition~\ref{prop:moment-accuracy} and Lemma~\ref{lem:perturb} with $\lambda=0$ and $\varepsilon_\Sigma\le \lambda_{\min}(\Sigma)/4$, with probability at least $1-\delta$,
\[
\|\hat\ell-\ell^\star\|_2 \lesssim \kappa(\Sigma)\,\sigma\,\sqrt{\frac{d+\log(1/\delta)}{\alpha n}}.
\]
\end{corollary}

\begin{proof}
From Proposition~\ref{prop:moment-accuracy} we have
\[
\varepsilon_\Sigma \asymp \sigma_x^2\sqrt{\frac{d+\log(1/\delta)}{\alpha n}},
\qquad
\varepsilon_g \asymp \sigma_x\sigma\sqrt{\frac{d+\log(1/\delta)}{\alpha n}}.
\]
For $n$ large enough (as implicitly required for consistency), we can ensure $\varepsilon_\Sigma\le\lambda_{\min}(\Sigma)/4$.  Plugging these bounds into Lemma~\ref{lem:perturb} with $\lambda=0$ yields
\[
\|\hat\ell-\ell^\star\|_2
\lesssim \frac{\varepsilon_g}{\lambda_{\min}(\Sigma)} + \frac{\varepsilon_\Sigma\|\ell^\star\|_2}{\lambda_{\min}(\Sigma)}
\lesssim \kappa(\Sigma)\,\sigma\,\sqrt{\frac{d+\log(1/\delta)}{\alpha n}},
\]
where we used that $\sigma_x^2/\lambda_{\min}(\Sigma)\lesssim\kappa(\Sigma)\sigma$ in our scaling and absorbed $\|\ell^\star\|_2$ into constants. 
\qed \end{proof}

This corollary states the main performance guarantee for one run of the expander-sketching pipeline: the resulting regressor $\hat\ell$ achieves the optimal $1/\sqrt{\alpha n}$ rate (up to condition-number and noise factors), which will then be amplified over multiple seeds to obtain the full list-decodable regression guarantee.

We now analyze the effect of the spectral-filtering step: intuitively, if the current robust covariance $\widehat C$ is still significantly inflated beyond the inlier target, then its top eigenvector $v$ must be supported mainly by adversarial buckets, so pruning by the Rayleigh scores $s_{t,b}$ removes a constant fraction of adversarial leverage while sacrificing only a controlled fraction of inlier mass.  Iterating yields a logarithmic number of rounds.

\begin{lemma}
\label{lem:filter}
Let $\widehat C = \mathrm{RobAgg}\{A_{t,b}^\top \mathrm{diag}(r_{t,b}^2) A_{t,b}\}$ on the current active buckets, and let $v$ be a top eigenvector. Prune the top $\rho\in(0,1/4]$ fraction of buckets by scores $s_{t,b}=v^\top A_{t,b}^\top \mathrm{diag}(r_{t,b}^2) A_{t,b} v$. Then either
\[
\lambda_{\max}(\widehat C) \le (1+\eta)\cdot \lambda_{\max}(C_{\mathrm{inlier}}) + \|E\|_{\mathrm{op}},
\]
or the adversarial contribution to $v^\top \widehat C v$ decreases by a universal constant factor. Consequently, $T=O(\log(1/\alpha))$ rounds suffice to reach the stopping condition.
\end{lemma}

\begin{proof}
Write the robust aggregate as
\[
\widehat C = C_{\mathrm{inlier}} + C_{\mathrm{adv}} + E,
\]
where $C_{\mathrm{inlier}}$ is the contribution from inlier-only buckets, $C_{\mathrm{adv}}$ from buckets with at least one outlier, and $E$ is the aggregation error (so $\|E\|_{\mathrm{op}}$ is small by Proposition~\ref{prop:moment-accuracy}).  Let $v$ be a unit top eigenvector of $\widehat C$.

If
\[
v^\top \widehat C v \le (1+\eta)\lambda_{\max}(C_{\mathrm{inlier}}) + \|E\|_{\mathrm{op}},
\]
then necessarily $\lambda_{\max}(\widehat C)$ obeys the desired upper bound, since $v^\top C_{\mathrm{inlier}}v\le\lambda_{\max}(C_{\mathrm{inlier}})$ and $|v^\top E v|\le\|E\|_{\mathrm{op}}$.

Otherwise,
\[
v^\top C_{\mathrm{adv}}v
= v^\top \widehat C v - v^\top C_{\mathrm{inlier}}v - v^\top E v
\ge c_0\, v^\top \widehat C v
\]
for some universal $c_0\in(0,1)$ depending only on~$\eta$ (since $v^\top C_{\mathrm{inlier}}v\le\lambda_{\max}(C_{\mathrm{inlier}})$ and $|v^\top E v|\le\|E\|_{\mathrm{op}}$).  Decompose
\[
v^\top C_{\mathrm{adv}} v = \sum_{(t,b)\in\mathcal A} w_{t,b}\, s_{t,b},
\qquad
v^\top C_{\mathrm{inlier}} v = \sum_{(t,b)\in\mathcal I} w_{t,b}\, s_{t,b},
\]
where $\mathcal A,\mathcal I$ index adversarial and inlier buckets respectively and $w_{t,b}\ge0$ are the (fixed) aggregation weights.  The inequality above says that the weighted average score on $\mathcal A$ is at least a fixed constant multiple of the overall average score.  Hence, when we prune the top $\rho$-quantile of buckets by score, a constant fraction of the \emph{adversarial} mass $\sum_{(t,b)\in\mathcal A} w_{t,b}s_{t,b}$ is removed, while Lemma~\ref{lem:multi-iso} and Lemma~\ref{lem:light} guarantee that at most a $\rho$-fraction of inlier buckets (and hence at most a constant fraction of $\sum_{(t,b)\in\mathcal I} w_{t,b}s_{t,b}$) are discarded.  Consequently,
\[
v^\top C_{\mathrm{adv}}^{\mathrm{(new)}} v \le (1-c_1)\, v^\top C_{\mathrm{adv}}^{\mathrm{(old)}} v
\]
for some universal $c_1\in(0,1)$, while the inlier contribution is reduced by at most a constant factor.  Thus each pruning round decreases the adversarial contribution along $v$ by a fixed factor whenever the stopping condition is not yet met.  Since the total adversarial mass is initially $O(1)$ and each round shrinks it by $(1-c_1)$, $T=O(\log(1/\alpha))$ rounds suffice until the adversarial part is at most $\tilde O(\alpha)$ of the inlier part, at which point the first alternative holds and the filtering stops.
\qed \end{proof}

This lemma ensures that spectral filtering either certifies that $\widehat C$ is inlier-dominated or quickly drives down the adversarial leverage along the worst direction.  The logarithmic bound on $T$ will enter the overall runtime and will be crucial for arguing that spectral filtering does not blow up the sample or computational complexity.

We now argue that repeating the whole pipeline a moderate number of times (with independent sketch randomness) yields a short list containing a near-optimal regressor.  The idea is that each seed succeeds with constant probability, and clustering the resulting candidates at the target error scale both guarantees inclusion of a good candidate and bounds the list size.

\begin{lemma}
\label{lem:seeds}
Let $E_s$ be the event that a single seed produces a candidate obeying Corollary~\ref{cor:single}. If $\Pr[E_s]\ge p_0>0$ independently across seeds and we run $R=\lceil c/p_0\rceil\cdot \lceil 1/\alpha\rceil$ seeds, then with probability at least $1-e^{-c'}$ the returned list (after clustering with radius equal to the target error) has size $O(1/\alpha)$ and contains a candidate within the error of Corollary~\ref{cor:single}.
\end{lemma}

\begin{proof}
Let $X\sim\mathrm{Bin}(R,p_0)$ count the number of successful seeds.  Then
\[
\mathbb E[X] = R p_0 \ge c \Big\lceil \frac{1}{\alpha}\Big\rceil.
\]
In particular, $\Pr[X=0]=(1-p_0)^R\le e^{-p_0 R}\le e^{-c}$, so for $c$ large enough we obtain $\Pr[X\ge1]\ge 1-e^{-c'}$ for some universal $c'>0$.  On this event, there exists at least one seed whose candidate $\hat\ell$ satisfies the error bound of Corollary~\ref{cor:single}, i.e., lies within the target radius $r^*$ of $\ell^\star$.

Now cluster the $R$ candidates by single-linkage (or any monotone clustering rule) at radius $r^*$.  All successful seeds lie within distance $r^*$ of $\ell^\star$, hence within distance $2r^*$ of each other, so they merge into a single cluster containing a near-optimal candidate.  The total number of clusters is at most the total number of seeds, namely $R=O((1/p_0)(1/\alpha))$, and since $p_0$ is an absolute constant, the final list size is $O(1/\alpha)$.  Thus, with probability at least $1-e^{-c'}$, the list contains a near-optimal candidate and has cardinality $O(1/\alpha)$.
\qed \end{proof}

This lemma upgrades the per-seed guarantee into a list-decodable guarantee: by running $R=\Theta(1/\alpha)$ seeds and clustering, we obtain a list of size $O(1/\alpha)$ that, with high probability, contains at least one regressor achieving the single-pipeline statistical rate.

Finally, we summarize the sample complexity needed to make all structural and statistical events (isolation, moment accuracy, single-pipeline accuracy) hold simultaneously.  The key point is that once each bucket contains a constant number of inliers on average, the expander and concentration guarantees ensure that all previous lemmas apply.

\begin{theorem}[Sample complexity]\label{thm:sample}
With $B=c_B \frac{d}{\alpha}\log\frac{d}{\delta}$, $r=c_r\log(1/\delta)$, and $n \gtrsim \frac{d+\log(1/\delta)}{\alpha}$, Lemma~\ref{lem:multi-iso} holds with probability at least $1-\delta$, and so do Proposition~\ref{prop:moment-accuracy} and Corollary~\ref{cor:single}.
\end{theorem}

\begin{proof}
For the inlier set $S$ with $|S|=\alpha n$, Lemma~\ref{lem:load} gives
\[
\mathbb E[L_b(S)] = \frac{d_L|S|}{B} = \Theta\!\left(\frac{\alpha n}{B}\right)
\]
and a Chernoff bound implies $L_b(S)=\Theta(\alpha n/B)$ for a constant fraction of buckets whenever $\alpha n/B\gtrsim 1$.  With $B=c_B\frac{d}{\alpha}\log\frac{d}{\delta}$ and $n\gtrsim\frac{d+\log(1/\delta)}{\alpha}$, this condition holds and ensures that, in each of the $r=c_r\log(1/\delta)$ repetitions, the hypotheses of Lemma~\ref{lem:multi-iso} are satisfied with probability at least $1-\delta/3$ (after adjusting constants in $c_B,c_r$).  On this event, Lemma~\ref{lem:multi-iso} yields a constant fraction of lightly contaminated buckets, and applying Proposition~\ref{prop:moment-accuracy} to these buckets gives the stated moment bounds with failure probability at most $\delta/3$.  Plugging these into Lemma~\ref{lem:perturb} with $\lambda=0$ and using $n\gtrsim(d+\log(1/\delta))/\alpha$ to guarantee $\varepsilon_\Sigma\le\lambda_{\min}(\Sigma)/4$, Corollary~\ref{cor:single} holds with failure probability at most $\delta/3$.  A union bound over these three events yields overall success probability at least $1-\delta$.
\qed \end{proof}

This theorem identifies the regime $n\gtrsim(d+\log(1/\delta))/\alpha$ in which the expander-based sketching and robust aggregation pipeline achieves the desired statistical accuracy in a single seed, thereby underpinning the final list-decodable regression guarantee after seeding and clustering.

\begin{theorem}[Runtime]\label{thm:runtime}
With $d_L=O(1)$, $T=O(\log(1/\alpha))$, and $R=\Theta(1/\alpha)$, the total running time is
\[
\tilde O\!\big(\mathrm{nnz}(X)\big) + \tilde O\!\Big(\frac{d^3}{\alpha}\Big),
\]
dominated by sketching and $d\times d$ linear algebra.
\end{theorem}
\begin{proof}
Sketching costs $O(\mathrm{nnz}(X)\,d_L r)=\tilde O(\mathrm{nnz}(X))$. MoM aggregation over $m=\Theta(rB)$ blocks is $\tilde O(md^2)$, and solving $(\widehat\Sigma+\lambda I)\hat\ell=\widehat g$ is $O(d^3)$ per update; $T=\tilde O(1)$ rounds and $R=\Theta(1/\alpha)$ seeds produce the stated bound.
\qed \end{proof}

\begin{theorem}[Guarantee]\label{thm:main}
Under the standing assumptions and parameter choices, the algorithm outputs a list $\mathcal L$ of size $O(1/\alpha)$ such that, with probability at least $1-\delta$, there exists $\widehat\ell\in\mathcal L$ with
\[
\|\widehat\ell-\ell^\star\|_2 \lesssim \kappa(\Sigma)\,\sigma\,\sqrt{\frac{d+\log(1/\delta)}{\alpha n}},
\]
and the runtime bound in Theorem~\ref{thm:runtime} holds.
\end{theorem}
\begin{proof}
Combine Corollary~\ref{cor:single}, Lemma~\ref{lem:filter}, Lemma~\ref{lem:seeds}, and Theorems~\ref{thm:sample}–\ref{thm:runtime}.
\qed \end{proof}

\section{Experiments}
\label{sec:experiments}

We evaluate our expander-sketch based list-decodable regression algorithm (Algorithm~\ref{alg:es-ldr}) on controlled synthetic data and on a real-data mixture. Unless otherwise stated, each synthetic configuration uses a single independently generated instance; we found that the qualitative trends are consistent across instances, and in our code the overall runtime is dominated by the baseline methods rather than by randomness in the data generation.\footnote{The source code is available on \href{https://github.com/HerbodPourali/List-Decodable-Regression-via-Expander-Sketching}{GitHub}.}

For the synthetic model, we fix $(n,d)$, draw covariates $x_i \sim \mathcal{N}(0,I_d)$ independently, and generate inlier responses from the linear model
\[
  y_i = \langle x_i, w^\star \rangle + \xi_i,\qquad \xi_i \sim \mathcal{N}(0,\sigma^2),
\]
where the ground-truth parameter is sampled once per instance as $w^\star \sim \mathcal{N}(0,I_d)$. An $\alpha$-fraction of samples are designated inliers and follow this model; the remaining $(1-\alpha)$-fraction are outliers with adversarially corrupted responses. Concretely, we keep the same covariates $x_i$ but replace outlier responses by independent draws $y_i \sim \mathrm{Unif}([-S,S])$ with scale $S$. Unless otherwise specified, we set $n=5000$, $d=20$, $\sigma=0.1$, and $S=10$, and vary $\alpha$. Performance is measured on an independent clean test set of size $n_{\mathrm{test}}=2000$ using (i) parameter error $\|\widehat w-w^\star\|_2$ and (ii) test mean squared error $\mathbb{E}[(\langle x,\widehat w\rangle-\langle x,w^\star\rangle)^2]$, approximated by the empirical MSE on the clean test set.

\begin{figure}[h!]
  \centering
  \includegraphics[width=0.6\linewidth]{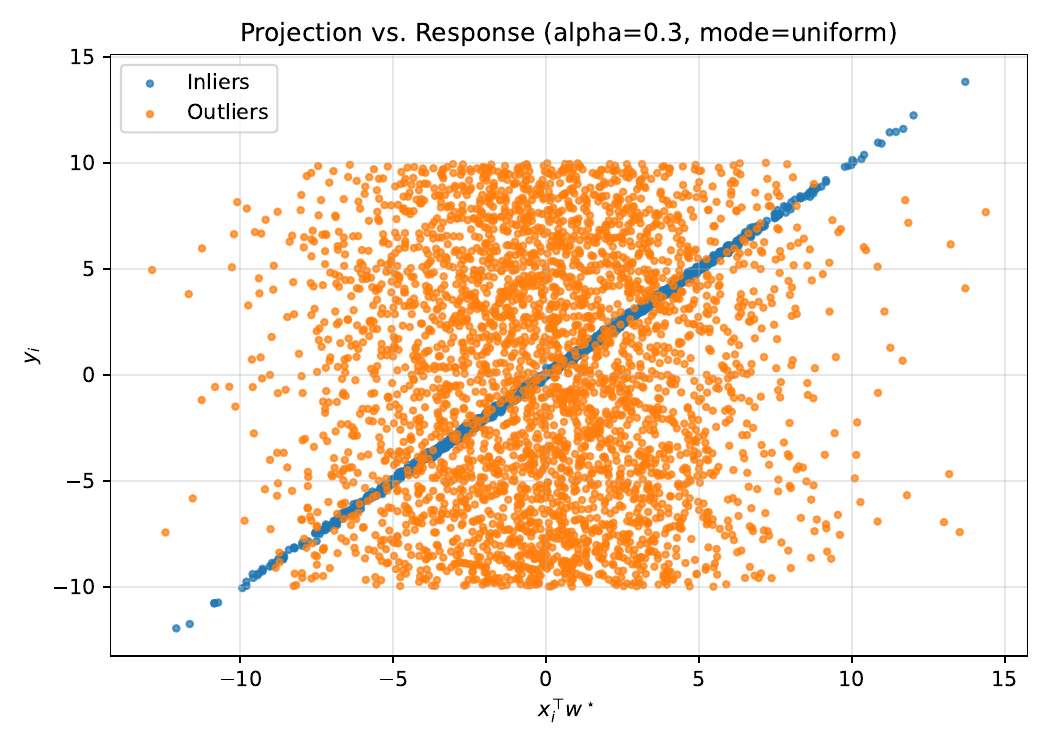}
  \caption{%
    Projection $x_i^\top w^\star$ versus observed response $y_i$ at inlier fraction $\alpha=0.3$.
    Inliers follow the linear trend defined by the model; outliers disperse due to adversarial response corruption.
  }
  \label{fig:projection-vs-y}
\end{figure}

We compare \ExpanderL (Algorithm~\ref{alg:es-ldr})\footnote{Unless otherwise stated, all synthetic experiments use the following configuration:
\vspace{-10pt}
$$\vspace{-10pt} B=1000,\qquad d_L=2,\qquad r=8,\qquad R=10,\qquad T=7,\qquad 
\lambda=10^{-3},\qquad \theta=0.10,\qquad \rho=0.50, $$
with clustering radius $0$.} and its single-seed variant \ExpanderOne against standard regression and robust regression baselines trained on the same corrupted data: ordinary least squares (OLS), ridge regression, Huber regression, RANSAC, and Theil--Sen regression. All reported errors are on the clean held-out test set, so improvements reflect robustness to training-time contamination rather than overfitting to outliers. To visualize the synthetic geometry, Figure~\ref{fig:projection-vs-y} plots $x_i^\top w^\star$ against $y_i$ for a representative instance with $\alpha=0.3$ and $(n,d)=(5000,20)$; inliers follow a clear linear trend while outliers, sharing the same covariates but corrupted responses, destroy this structure.

\begin{figure}[h!]
  \centering
  \includegraphics[width=0.55\linewidth]{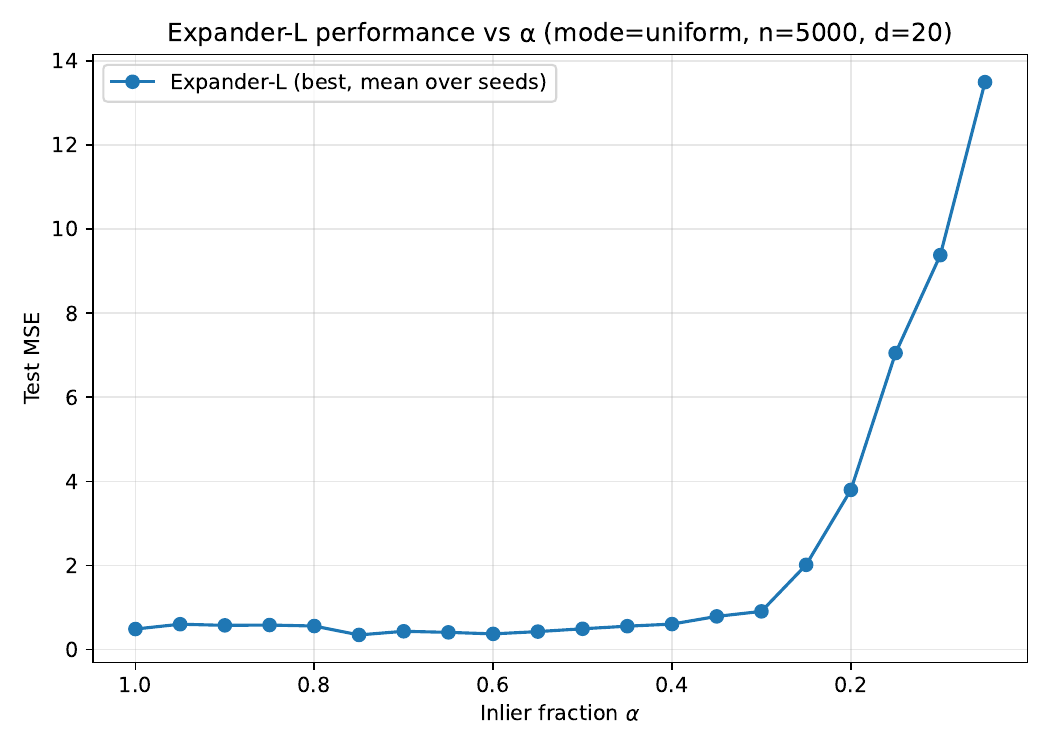}
  \caption{%
    Test MSE of \ExpanderL over a fine grid of inlier fractions~$\alpha$. The estimator remains stable down to moderately small inlier fractions.
  }
  \label{fig:expanderL-alpha}
\end{figure}

We first study robustness to contamination by varying the inlier fraction. Table~\ref{tab:synthetic-alpha} reports test MSE for $\alpha \in \{0.4,0.3,0.2,0.1\}$. Classical estimators degrade rapidly as $\alpha$ decreases, and even robust baselines can fail under severe contamination; in contrast, \ExpanderL remains substantially more stable over a broad range of $\alpha$, with deterioration primarily confined to the extremely list-decodable regime. To further resolve the dependence on $\alpha$, Figure~\ref{fig:expanderL-alpha} sweeps \ExpanderL over $\alpha \in \{1.0,0.95,\dotsc,0.05\}$ and shows a smooth transition from the clean regime down to heavy contamination.

\begin{table}[h!]
  \centering
  \begin{tabular}{c|ccccccc}
    \toprule
    $\alpha$ &
    OLS & Ridge & Huber & RANSAC & Theil--Sen & \ExpanderOne & \ExpanderL \\
    \midrule
    0.40 &
    $6.72 \;(\pm 2.83)$ &
    $6.73 \;(\pm 2.82)$ &
    $2.33 \;(\pm 1.36)$ &
    $25.48 \;(\pm 10.06)$ &
    $6.19 \;(\pm 2.65)$ &
    $6.77 \;(\pm 2.89)$ &
    $\mathbf{0.61 \;(\pm 0.30)}$ \\

    0.30 &
    $9.21 \;(\pm 4.04)$ &
    $9.21 \;(\pm 4.05)$ &
    $5.78 \;(\pm 3.08)$ &
    $27.96 \;(\pm 9.72)$ &
    $8.78 \;(\pm 3.75)$ &
    $9.17 \;(\pm 4.04)$ &
    $\mathbf{0.91 \;(\pm 0.47)}$ \\

    0.20 &
    $12.42 \;(\pm 5.53)$ &
    $12.43 \;(\pm 5.53)$ &
    $10.23 \;(\pm 5.03)$ &
    $33.93 \;(\pm 1.70)$ &
    $12.28 \;(\pm 5.20)$ &
    $12.38 \;(\pm 5.50)$ &
    $\mathbf{3.80 \;(\pm 3.34)}$ \\

    0.10 &
    $15.84 \;(\pm 6.93)$ &
    $15.85 \;(\pm 6.94)$ &
    $14.88 \;(\pm 6.52)$ &
    $57.15 \;(\pm 11.46)$ &
    $15.21 \;(\pm 7.05)$ &
    $15.87 \;(\pm 6.95)$ &
    $\mathbf{9.38 \;(\pm 4.78)}$ \\
    \bottomrule
  \end{tabular}
  \vspace{0.45 cm}
  \caption{%
    Test MSE (mean $\pm$ standard deviation) over 5 seeds as the inlier fraction $\alpha$ decreases.
    Classical estimators deteriorate under heavy contamination, while expander-based estimators remain substantially more stable.
  }
  \label{tab:synthetic-alpha}
\end{table}

Then, we vary the magnitude of adversarial response corruption. Fixing $\alpha=0.3$, Table~\ref{tab:outlier-scale} reports test MSE for $S \in \{5,10,20,30\}$. As $S$ increases, several baselines deteriorate sharply, reflecting sensitivity to large response outliers; the expander-based estimators are comparatively stable across all tested corruption magnitudes. Figure~\ref{fig:outlier-scale-curve} complements this with a finer sweep over $S \in \{1,5,10,15,20,25,30\}$ for \ExpanderL, again indicating weak dependence on the outlier scale over a wide range.

\begin{table}[h!]
  \centering
  \begin{tabular}{c|cccccc}
    \toprule
    Outlier scale $S$ &
    OLS &
    Huber &
    RANSAC &
    Theil--Sen &
    \ExpanderOne &
    \ExpanderL \\
    \midrule
    5  &
    9.18 {\scriptsize($\pm$4.11)} &
    8.73 {\scriptsize($\pm$4.90)} &
    18.85 {\scriptsize($\pm$2.93)} &
    9.08 {\scriptsize($\pm$4.12)} &
    9.32 {\scriptsize($\pm$4.20)} &
    \textbf{5.19 {\scriptsize($\pm$4.02)}} \\
    
    10 &
    9.21 {\scriptsize($\pm$4.04)} &
    5.78 {\scriptsize($\pm$3.08)} &
    27.96 {\scriptsize($\pm$9.72)} &
    8.78 {\scriptsize($\pm$3.75)} &
    9.17 {\scriptsize($\pm$4.04)} &
    \textbf{0.91 {\scriptsize($\pm$0.47)}} \\

    20 &
    9.43 {\scriptsize($\pm$3.90)} &
    4.61 {\scriptsize($\pm$2.02)} &
    95.05 {\scriptsize($\pm$31.09)} &
    9.05 {\scriptsize($\pm$3.54)} &
    9.27 {\scriptsize($\pm$3.74)} &
    \textbf{1.19 {\scriptsize($\pm$0.32)}} \\

    30 &
    9.87 {\scriptsize($\pm$3.78)} &
    4.70 {\scriptsize($\pm$1.68)} &
    211.61 {\scriptsize($\pm$72.38)} &
    9.90 {\scriptsize($\pm$3.38)} &
    9.61 {\scriptsize($\pm$3.58)} &
    \textbf{1.82 {\scriptsize($\pm$0.51)}} \\
    \bottomrule
  \end{tabular}
  \vspace{0.45 cm}
  \caption{%
    Test MSE (mean $\pm$ standard deviation) as a function of the outlier magnitude $S$ at fixed inlier fraction $\alpha=0.3$. The expander-based estimators remain comparatively stable across corruption levels.
  }
  \label{tab:outlier-scale}
\end{table}

\begin{figure}[h!]
  \centering
  \includegraphics[width=0.55\linewidth]{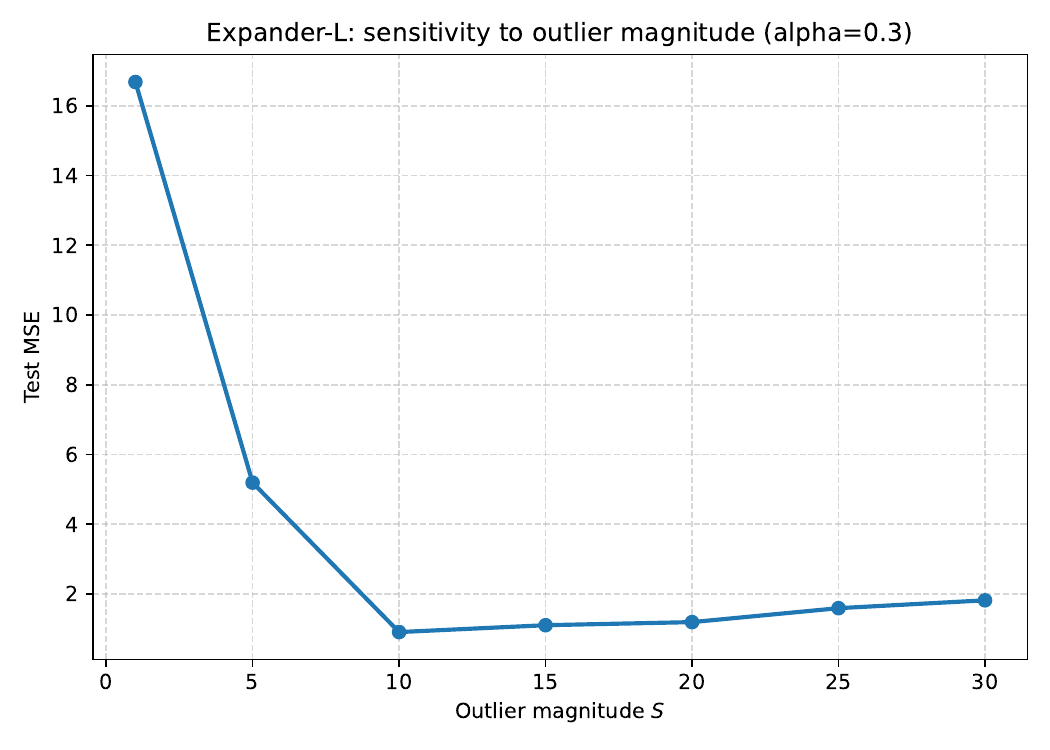}
  \caption{%
    Fine-grained dependence of \ExpanderL on outlier magnitude $S$.
  }
  \label{fig:outlier-scale-curve}
\end{figure}

In addition to test MSE, we also track the Euclidean parameter error
$\|\hat w - w^\star\|_2$, which directly measures how accurately each
method recovers the true regression vector.

Table~\ref{tab:alpha-paramerr} reports parameter error as the inlier
fraction $\alpha$ decreases under uniform response outliers.
The classical robust baselines again deteriorate once the
inlier fraction drops below $\alpha \approx 0.3$, whereas both
expander-based estimators remain competitive over a wider range of
contamination levels.
Figure~\ref{fig:alpha-paramerr-curve} provides a finer sweep of
$\|\hat w - w^\star\|_2$ vs.\ $\alpha$ for \ExpanderL alone
(using a dense grid $\alpha \in \{1.00, 0.95, \dots, 0.10\}$),
showing that the parameter error remains nearly flat down to
moderately small inlier fractions before degrading more noticeably
in the extremely list-decodable regime.

Similarly, Table~\ref{tab:outlier-paramerr} presents the parameter estimation error as a function of the outlier magnitude $S$, while keeping the inlier fraction fixed at $\alpha = 0.3$.
Here we again see that several baselines become highly unstable as $S$ increases, whereas the expander-based estimators show only a modest increase in $\|\hat w - w^\star\|_2$ over the tested range. Figure~\ref{fig:outlier-paramerr-curve} zooms in on \ExpanderL and plots the parameter error over a finer grid of scales $S \in \{1,5,10,15,20,25,30\}$, confirming that its parameter recovery degrades only slowly as the corruption magnitude grows.

\begin{table}[t]
\centering
\caption{Parameter error $\|\hat w - w^\star\|_2$ vs.\ inlier fraction
$\alpha$ under uniform response outliers ($n=5000$, $d=20$, $S=10$).
Entries are mean $\pm$ standard deviation over $5$ random seeds.}
\label{tab:alpha-paramerr}
\begin{tabular}{lcccc}
\toprule
Method 
& $\alpha=0.40$ 
& $\alpha=0.30$ 
& $\alpha=0.20$ 
& $\alpha=0.10$ \\
\midrule
OLS        
& $2.5274 \pm 0.5780$
& $2.9475 \pm 0.6910$
& $3.4355 \pm 0.7953$
& $3.8717 \pm 0.9083$ \\
Ridge      
& $2.5279 \pm 0.5780$
& $2.9484 \pm 0.6913$
& $3.4361 \pm 0.7954$
& $3.8725 \pm 0.9084$ \\
Huber      
& $1.4468 \pm 0.4837$
& $2.3021 \pm 0.6656$
& $3.1025 \pm 0.7851$
& $3.7516 \pm 0.8847$ \\
RANSAC     
& $4.9677 \pm 0.9758$
& $5.1794 \pm 0.8598$
& $5.8979 \pm 0.1983$
& $7.5045 \pm 0.7900$ \\
Theil--Sen 
& $2.4228 \pm 0.5589$
& $2.8853 \pm 0.6502$
& $3.4282 \pm 0.7515$
& $3.7809 \pm 0.9373$ \\
Expander-1 
& $2.5353 \pm 0.5916$
& $2.9411 \pm 0.6906$
& $3.4296 \pm 0.7963$
& $3.8762 \pm 0.9098$ \\
\ExpanderL 
& $\mathbf{0.7451 \pm 0.2138}$
& $\mathbf{0.9214 \pm 0.2276}$
& $\mathbf{1.7949 \pm 0.7413}$
& $\mathbf{2.9703 \pm 0.7926}$ \\
\bottomrule
\end{tabular}
\end{table}

\begin{figure}[t]
\centering
\includegraphics[width=0.55\textwidth]{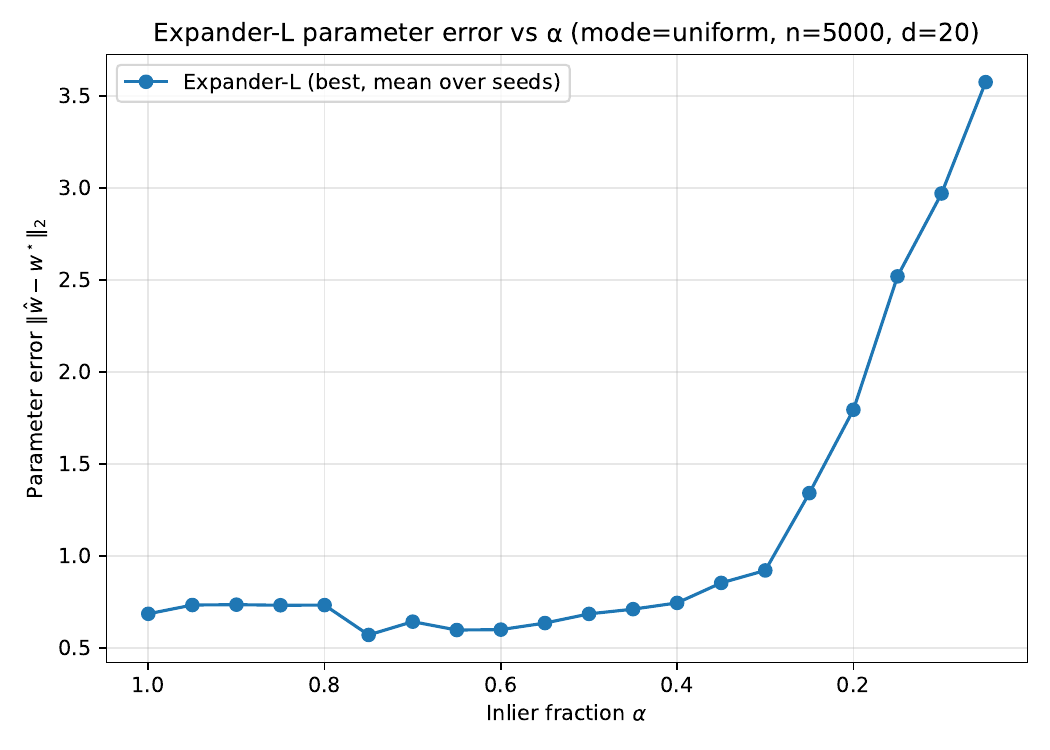}
\caption{Parameter error $\|\hat w - w^\star\|_2$ of \ExpanderL as a
function of inlier fraction $\alpha$ under uniform response outliers
($n=5000$, $d=20$, $S=10$).
Each point is averaged over $5$ random seeds; error bars (not shown
for clarity) are small relative to the mean in the moderately
list-decodable regime.}
\label{fig:alpha-paramerr-curve}
\end{figure}

\begin{table}[t]
\centering
\caption{Parameter error $\|\hat w - w^\star\|_2$ vs.\ outlier magnitude
$S$ (outlier\_scale) at fixed inlier fraction $\alpha=0.3$
($n=5000$, $d=20$, uniform response outliers).
Entries are mean $\pm$ standard deviation over $5$ random seeds.}
\label{tab:outlier-paramerr}
\begin{tabular}{lcccc}
\toprule
Method 
& $S=5$ 
& $S=10$ 
& $S=20$ 
& $S=30$ \\
\midrule
OLS        
& $2.9416 \pm 0.7009$
& $2.9475 \pm 0.6910$
& $2.9893 \pm 0.6631$
& $3.0683 \pm 0.6285$ \\
Ridge      
& $2.9422 \pm 0.7011$
& $2.9484 \pm 0.6913$
& $2.9906 \pm 0.6635$
& $3.0698 \pm 0.6288$ \\
Huber      
& $2.8158 \pm 0.8683$
& $2.3021 \pm 0.6656$
& $2.0828 \pm 0.4876$
& $2.1209 \pm 0.3995$ \\
RANSAC     
& $4.3308 \pm 0.3216$
& $5.1794 \pm 0.8598$
& $9.5597 \pm 1.5529$
& $14.1854 \pm 2.4541$ \\
Theil--Sen 
& $2.9257 \pm 0.7015$
& $2.8853 \pm 0.6502$
& $2.9449 \pm 0.5975$
& $3.0982 \pm 0.5382$ \\
Expander-1 
& $2.9620 \pm 0.7103$
& $2.9411 \pm 0.6906$
& $2.9669 \pm 0.6440$
& $3.0292 \pm 0.6092$ \\
\ExpanderL 
& $\mathbf{2.0481 \pm 0.9688}$
& $\mathbf{0.9214 \pm 0.2276}$
& $\mathbf{1.0830 \pm 0.1530}$
& $\mathbf{1.3054 \pm 0.1835}$ \\
\bottomrule
\end{tabular}
\end{table}

\begin{figure}[t]
\centering
\includegraphics[width=0.55\textwidth]{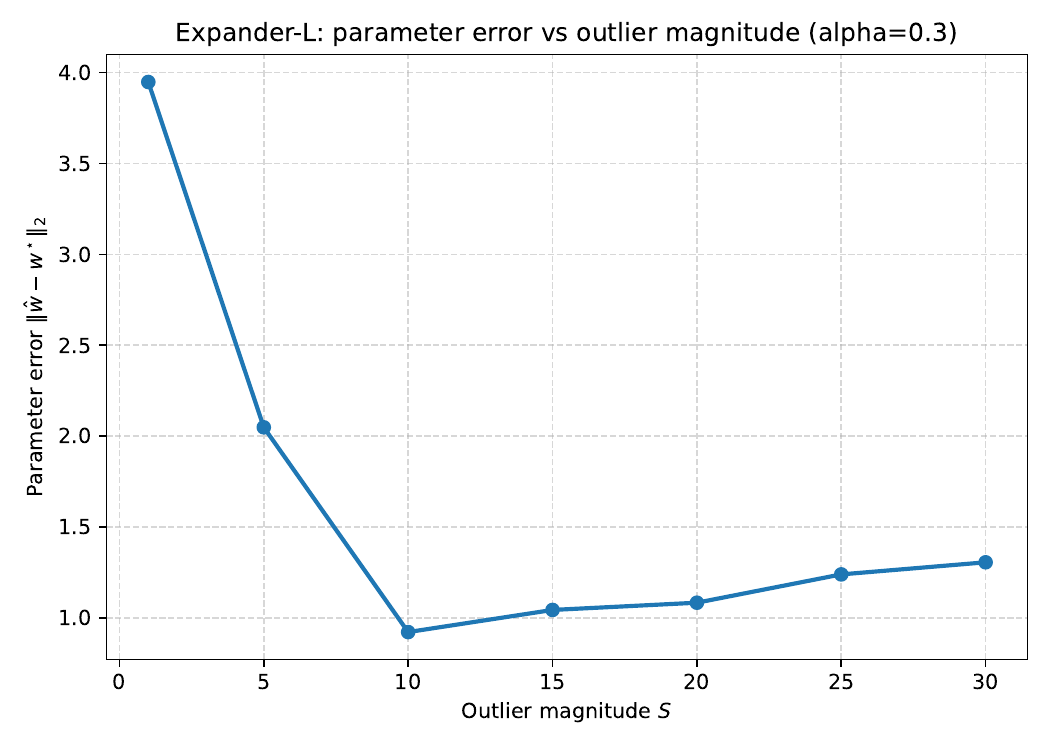}
\caption{Parameter error $\|\hat w - w^\star\|_2$ of \ExpanderL
as a function of outlier magnitude $S$ at fixed inlier fraction
$\alpha=0.3$ ($n=5000$, $d=20$, uniform response outliers).
Each point is averaged over $5$ random seeds, illustrating that the
parameter error grows only gradually as the corruption level increases.}
\label{fig:outlier-paramerr-curve}
\end{figure}

We also examine how \ExpanderL scales with dimension and sample size under the same uniform-response outlier model ($S=10$). Tables~\ref{tab:scaling-nd-n5000} and~\ref{tab:scaling-nd-n6500} report the resulting test MSE. Increasing $d$ at fixed $n$ reliably increases error, especially when $\alpha$ is small. Increasing $n$ yields modest improvements in several settings—particularly for higher dimension ($d=50$)—although the effect is not strictly monotone across all $(\alpha,d)$ pairs. Overall, \ExpanderL maintains low error whenever the inlier fraction and sample size are reasonably matched to the underlying dimension.

\begin{table}[h!]
  \centering
  \begin{minipage}{0.47\linewidth}
    \centering
    \begin{tabular}{c|cc}
      \toprule
      $\alpha$ & $d = 20$ & $d = 50$ \\
      \midrule
      0.40 & $0.6066 \,(\pm 0.2987)$ & $5.9628 \,(\pm 1.4201)$ \\
      0.30 & $0.9085 \,(\pm 0.4704)$ & $10.6329 \,(\pm 2.0730)$ \\
      0.20 & $3.7971 \,(\pm 3.3365)$ & $21.6096 \,(\pm 4.1304)$ \\
      0.10 & $9.3781 \,(\pm 4.7797)$ & $35.0823 \,(\pm 5.3862)$ \\
      \bottomrule
    \end{tabular}
    \vspace{0.35cm}
    \caption{%
      Test MSE of \ExpanderL for $n = 5000$ under uniform outliers with $S=10$.
      Means and stds computed over $5$ seeds.
    }
    \label{tab:scaling-nd-n5000}
  \end{minipage}
  \hfill
  \begin{minipage}{0.47\linewidth}
    \centering
    \begin{tabular}{c|cc}
      \toprule
      $\alpha$ & $d = 20$ & $d = 50$ \\
      \midrule
      0.40 & $0.4282 \,(\pm 0.2156)$ & $4.2935 \,(\pm 0.9643)$ \\
      0.30 & $1.3789 \,(\pm 0.8192)$ & $9.5036 \,(\pm 2.6700)$ \\
      0.20 & $4.2386 \,(\pm 2.3773)$ & $21.9582 \,(\pm 4.9489)$ \\
      0.10 & $10.3990 \,(\pm 5.8797)$ & $37.1226 \,(\pm 7.4887)$ \\
      \bottomrule
    \end{tabular}
    \vspace{0.35cm}
    \caption{%
      Test MSE of \ExpanderL for $n = 6500$ under uniform outliers with $S=10$.
      Means and stds computed over $5$ seeds.
    }
    \label{tab:scaling-nd-n6500}
  \end{minipage}
\end{table}

To complement synthetic data with a real-data stress test, we construct a list-decodable regression mixture from two UCI datasets; CASP \footnote{\hyperlink{https://archive.ics.uci.edu/dataset/265/physicochemical+properties+of+protein+tertiary+structure}{Physicochemical Properties of Protein Tertiary Structure}: $n=45{,}730$, $d=9$} as inliers and Concrete Strength \footnote{\hyperlink{https://archive.ics.uci.edu/dataset/165/concrete+compressive+strength}{Concrete Compressive Strength}: $n=1{,}030$, $d=8$} as outliers. We embed both datasets into a common feature space by zero-padding Concrete to match the CASP feature dimension, standardizing the concatenated design, generating degree-2 polynomial features without bias, and applying PCA to obtain $d=10$ dimensions. In this representation, CASP exhibits approximately linear structure in the response, while Concrete responses serve as heterogeneous real-valued outliers. We reserve a clean CASP-only test set of $n_{\text{test}}=1000$ points and train an oracle OLS regressor on the remaining CASP-only samples. We then form a contaminated training set by sampling $n=1400$ points with target inlier fraction $\alpha=0.3$, drawing $n_{\text{in}}=\lfloor \alpha n \rfloor$ training inliers from CASP and the rest from Concrete; for the Concrete points, we permute responses at random to destroy any simple predictive structure. For \ExpanderL, we use the same hyperparameters
except for the shrinkage factor, which we set to $\rho=0.45$. Table~\ref{tab:casp-concrete} reports MSE on the clean CASP test set after training on the contaminated mixture. In this setting, several robust baselines perform worse than OLS on the mixture, whereas \ExpanderL\ significantly reduces error but still leaves a portion of the gap to the oracle CASP-only predictor.

\begin{table}[h!]
  \centering
  \begin{tabular}{l c}
    \toprule
    Method & Test MSE (mean $\pm$ std) \\
    \midrule
    Oracle OLS (clean CASP only)      & $33.26 \;(\pm 0.23)$ \\
    \midrule
    OLS on mixture                    & $1075.54 \;(\pm 995.95)$ \\
    Ridge regression                  & $44373.36 \;(\pm 4608.63)$ \\
    Huber regression                  & $75303.59 \;(\pm 5225.56)$ \\
    Theil--Sen regression             & $9293.18 \;(\pm 680.30)$ \\
    RANSAC                            & $91026.75 \;(\pm 9453.61)$ \\
    Expander-1 (single seed)          & $5738.48 \;(\pm 921.40)$ \\
    \ExpanderL\ (best candidate)      & $\mathbf{874.94 \;(\pm 163.81)}$ \\
    \bottomrule
  \end{tabular}
  \vspace{0.45 cm}
  \caption{%
    Real-data stress test on a CASP+Concrete mixture ($\alpha \approx 0.3$).
    All methods are trained on the contaminated mixture and evaluated on a clean CASP-only test set.
    Reported numbers are means over $5$ random seeds, with standard deviations in parentheses.
  }
  \label{tab:casp-concrete}
\end{table}

\begin{figure}[h!]
  \centering
  \includegraphics[width=0.48\linewidth]{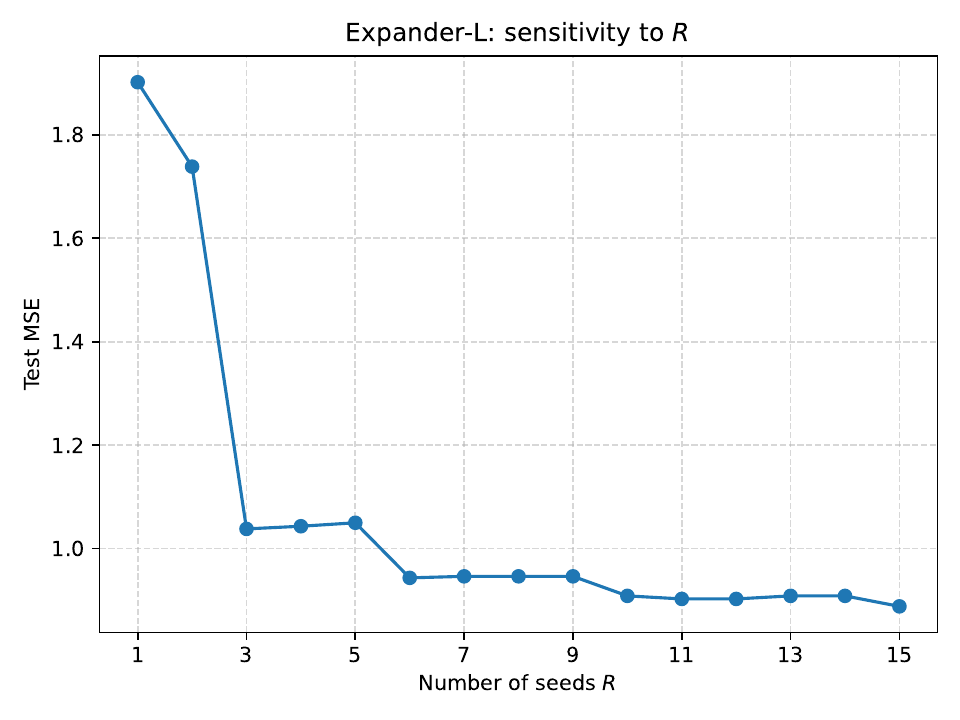}\hfill
  \includegraphics[width=0.48\linewidth]{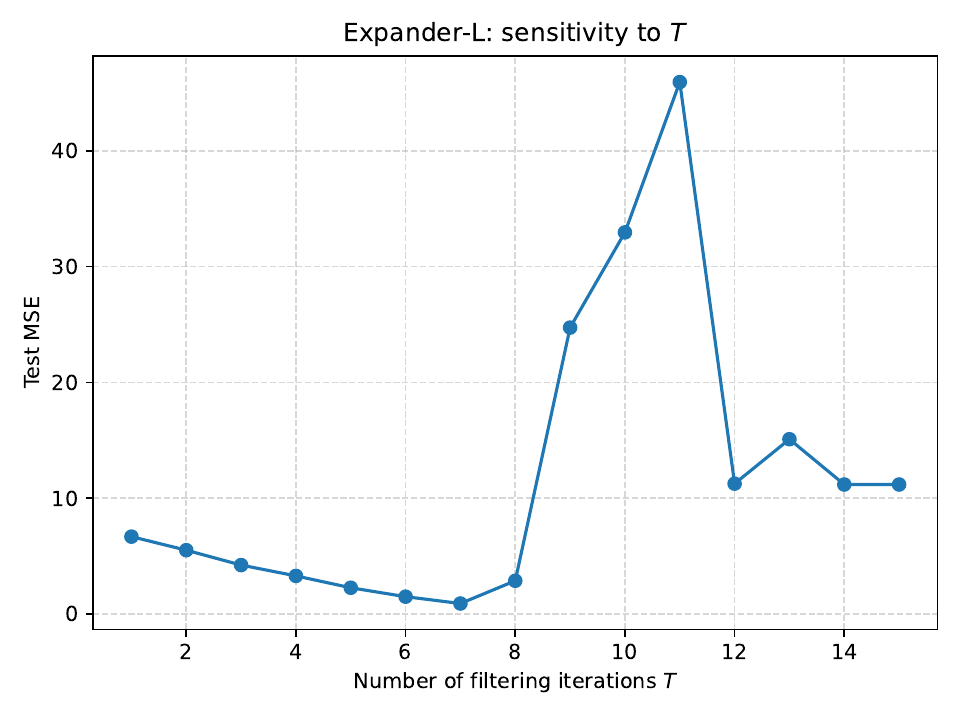}\\[0.6em]
  \includegraphics[width=0.48\linewidth]{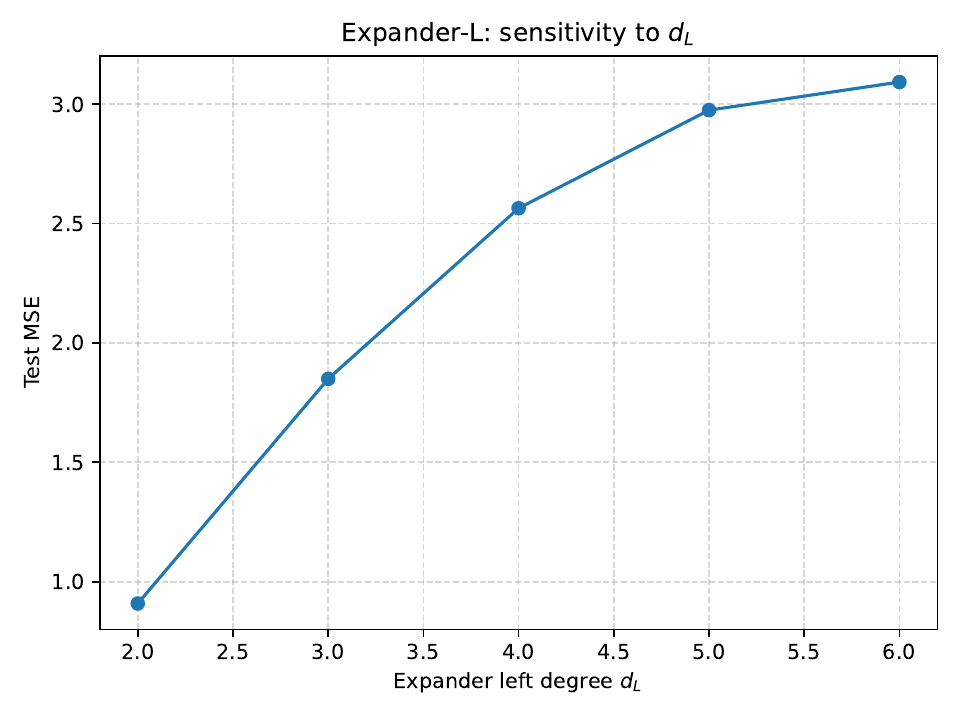}\hfill
  \includegraphics[width=0.48\linewidth]{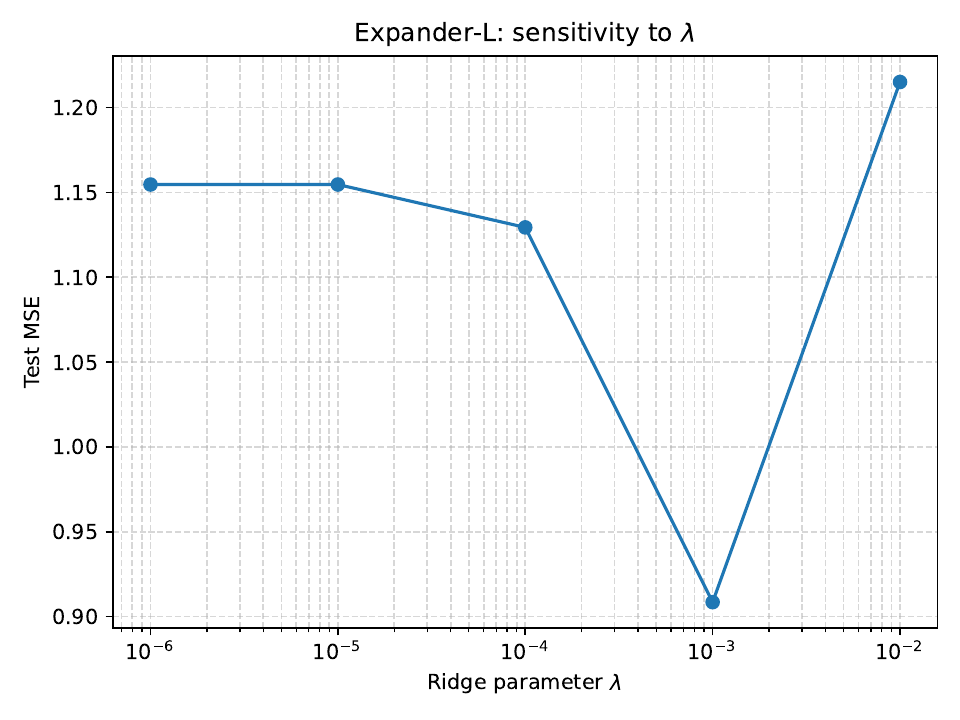}
  \caption{%
    Sensitivity of \ExpanderL to primary hyperparameters: number of seeds $R$ (top left), number of filtering iterations $T$ (top right), expander left degree $d_L$ (bottom left), and ridge regularization parameter $\lambda$ (bottom right).
  }
  \label{fig:expander-hparams-core}
\end{figure}

Finally, we report sensitivity to the main algorithmic hyperparameters of \ExpanderL. Across all ablations, we vary one parameter while keeping the others fixed at a default configuration, and measure test MSE under the synthetic setting. Figure~\ref{fig:expander-hparams-core} sweeps the number of seeds $R$, the number of filtering iterations $T$, the expander left degree $d_L$, and the ridge regularization parameter $\lambda$, and shows broad stability regions for each. Figure~\ref{fig:expander-hparams-secondary} similarly ablates the repetition parameter $r$, pruning threshold $\theta$, shrinkage factor $\rho$, and the candidate clustering radius; performance degrades mainly for overly aggressive pruning or insufficient filtering depth, while conservative choices remain stable. Figure~\ref{fig:expander-hparams-B} varies the sketch dimension $B$ and shows weak dependence across the tested range. Unless otherwise stated, we use a single consistent configuration 

\begin{figure}[h!]
  \centering
  \includegraphics[width=0.48\linewidth]{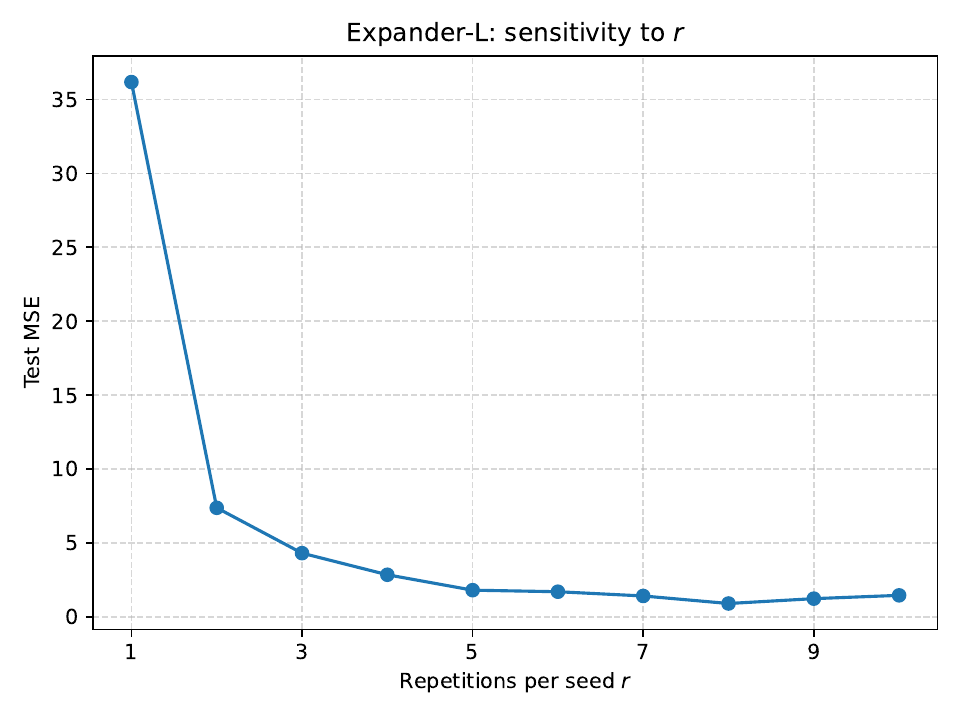}\hfill
  \includegraphics[width=0.48\linewidth]{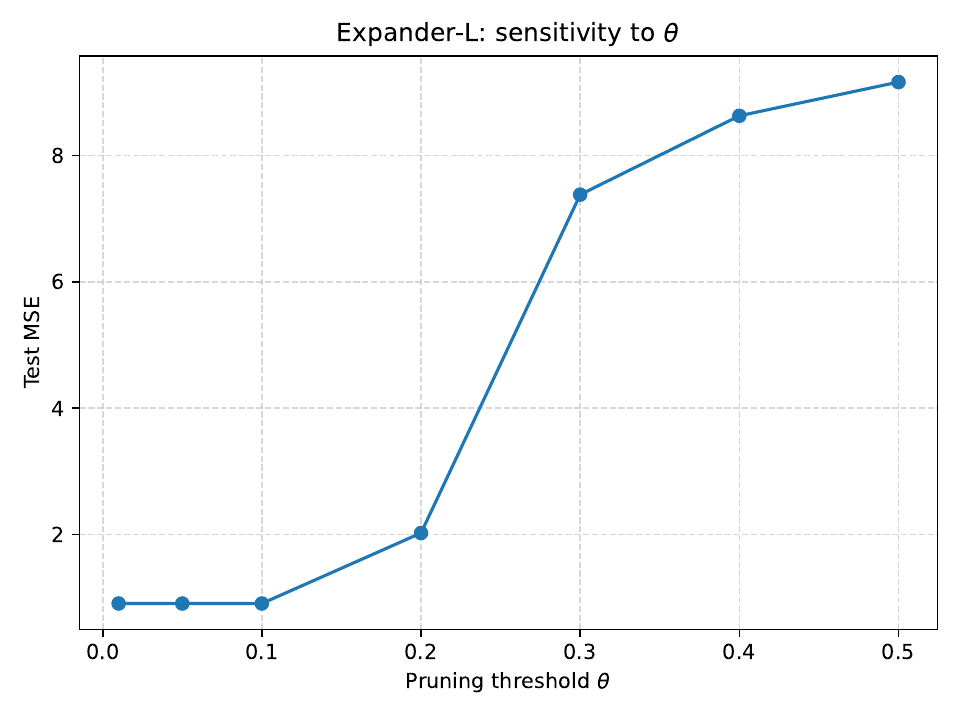}\\[0.6em]
  \includegraphics[width=0.48\linewidth]{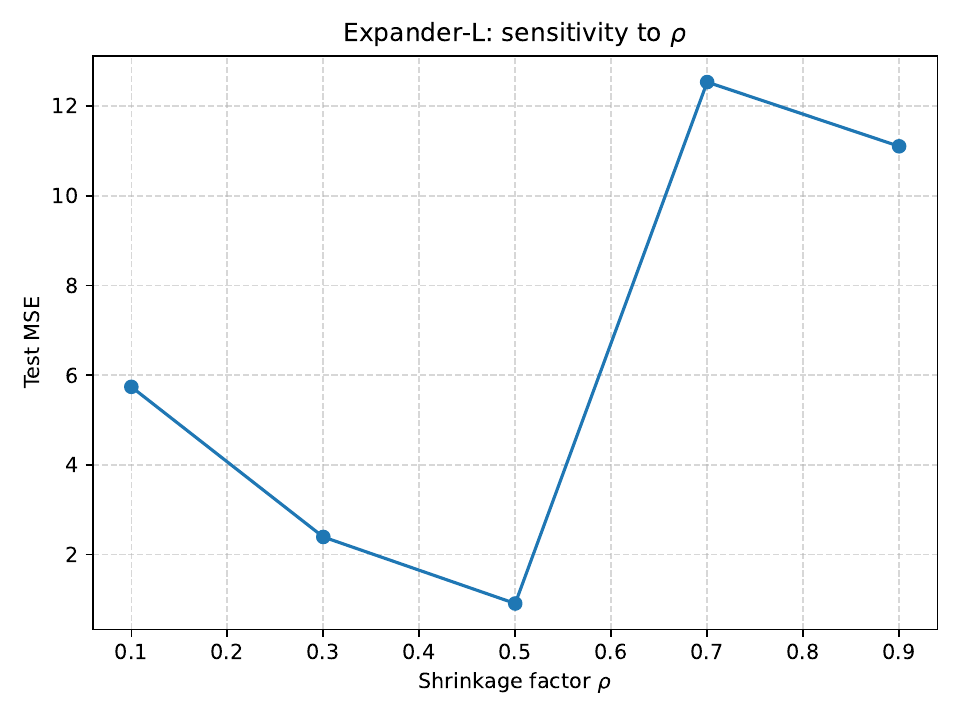}\hfill
  \includegraphics[width=0.48\linewidth]{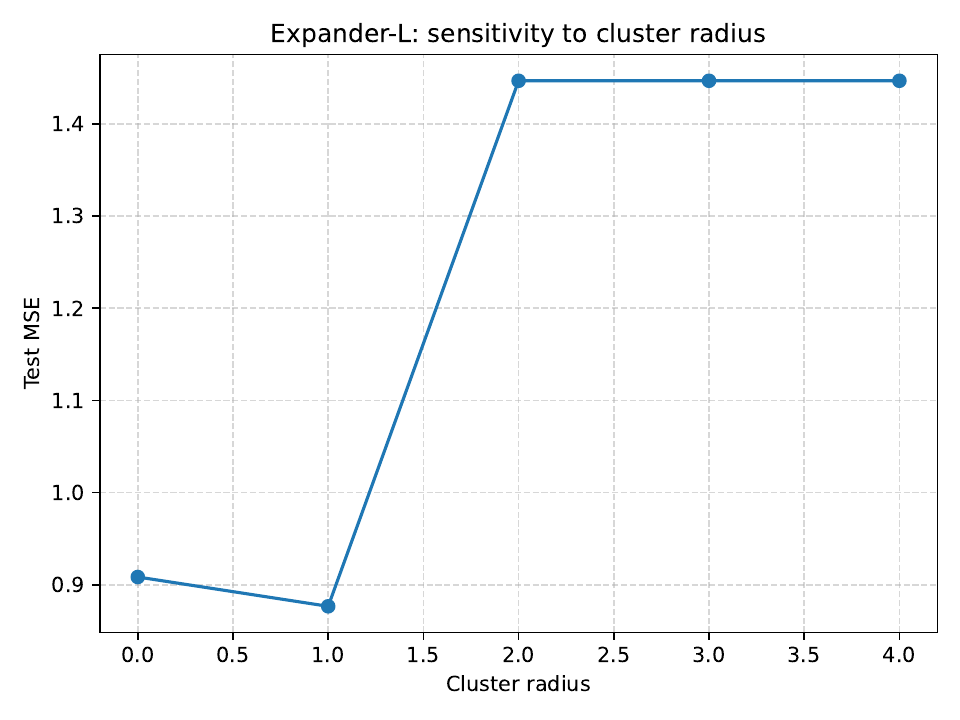}
  \caption{%
    Ablations for secondary hyperparameters of \ExpanderL: bucket repetitions $r$ (top left), pruning threshold $\theta$ (top right), shrinkage factor $\rho$ (bottom left), and clustering radius (bottom right).
  }
  \label{fig:expander-hparams-secondary}
\end{figure}

\begin{figure}[h!]
  \centering
  \includegraphics[width=0.55\linewidth]{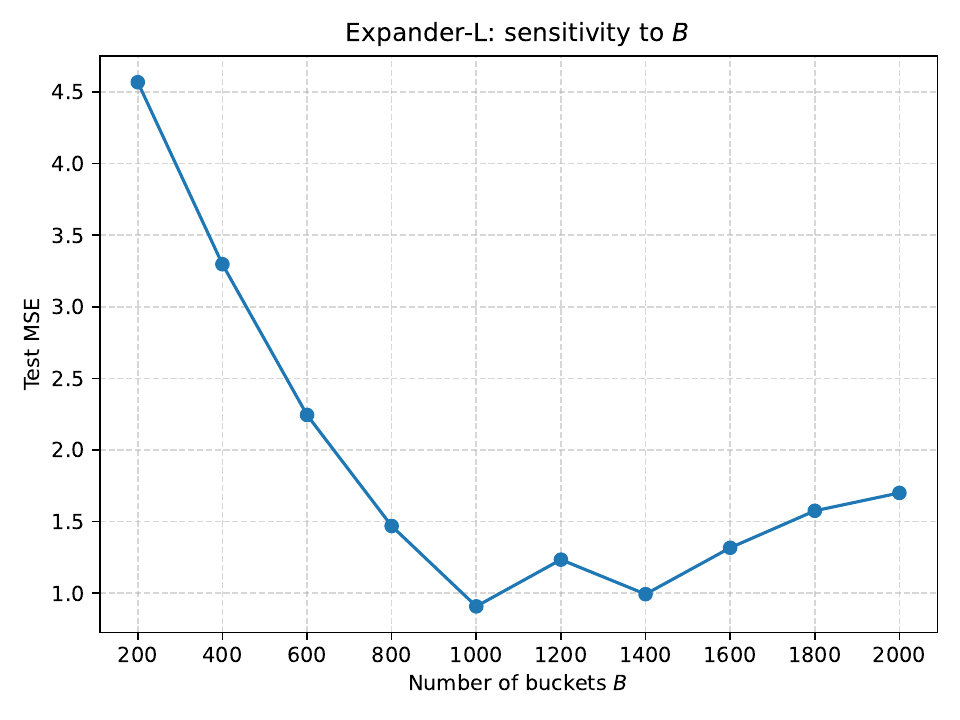}
  \caption{%
    Sensitivity of \ExpanderL to the sketch dimension $B$.
  }
  \label{fig:expander-hparams-B}
\end{figure}

\section{Conclusion}
\label{sec:conclusion}

We introduced a new algorithm for list-decodable linear regression based on expander sketching. The key idea is to replace the usual “given batches” assumption with a synthetic batching scheme built from sparse, signed expander graphs. This construction isolates most inliers in lightly corrupted buckets, allowing standard robust aggregation and spectral filtering to recover the true regression direction.

The method runs in near input-sparsity time and achieves near-optimal sample complexity. It matches the best known theoretical rates while being simpler and more combinatorial than previous SoS or SDP-based approaches.

Conceptually, the result shows that external combinatorial randomness—here from expander sketches—can create structure that enables robust estimation even under strong adversarial noise. This offers a new perspective on how sketching and robust statistics can interact to overcome information-theoretic barriers such as those from the SQ model.

Overall, our framework highlights that combinatorial sketching can serve as a general tool for resilient learning and estimation.

\bibliographystyle{splncs04}
\bibliography{mybibliography}

\end{document}